\documentclass[11pt]{article}
\usepackage[utf8]{inputenc}
\usepackage{amsfonts,amsthm,amsmath,amssymb,url}
\usepackage{xcolor}
\usepackage{graphicx}
\usepackage{fullpage}
\usepackage{thm-restate}
\usepackage{algorithmic}
\usepackage{algorithm}
\usepackage[T1]{fontenc}
\definecolor{mydarkblue}{rgb}{0,0.08,0.45}
\usepackage[colorlinks=true, linkcolor=mydarkblue, urlcolor=blue, citecolor=gray]{hyperref}
\usepackage{mathtools}
\usepackage{graphicx}
\usepackage{caption}

\usepackage{bm}
\usepackage{bbm}
\usepackage{booktabs}
\usepackage{nicefrac}
\usepackage{microtype}
\usepackage{xspace}

\usepackage[capitalize]{cleveref}
\usepackage{latexsym}
\usepackage{subfigure}
\usepackage{mathtools}
\usepackage{xcolor}
\usepackage{mleftright}
\usepackage[normalem]{ulem}
\usepackage{array}
\usepackage{makecell}
\usepackage{multicol}
\usepackage{enumitem}

\usepackage[numbers]{natbib}

\newtheorem{theorem}{Theorem}

\newtheorem{lemma}[theorem]{Lemma}
\newtheorem{defn}[theorem]{Definition}
\newtheorem{assp}{Assumption}
\newtheorem*{theorem*}{Theorem}

\usepackage[symbol]{footmisc}

\title{Random Gegenbauer Features for Scalable Kernel Methods}
\author{
Insu Han\footnotemark\\Yale University\\\texttt{insu.han@yale.edu}
\and
Amir Zandieh\footnotemark[1]\\MPI-Informatics\\\texttt{azandieh@mpi-inf.mpg.de}
\and
Haim Avron\\Tel Aviv University\\\texttt{haimav@tauex.tau.ac.il}
}

\date{\today}

\newcommand{\trace}{\mathrm{Tr}}

\newcommand{\rank}{\mathrm{rank}}

\newcommand{\norm}[1]{\ensuremath{\left\| #1 \right\|}}

\newcommand{\inner}[1]{\left \langle {#1} \right \rangle}
\newcommand{\bigo}{\mathcal{O}}
\newcommand{\abs}[1]{\left |#1\right|}

\newcommand{\nnz}[1]{\mathrm{nnz}\left(#1\right)}
\def\op{\mathrm{op}}

\def\diag{\mathop{\rm Diag}}

\newcommand{\wtphi}{%
	\mspace{2mu}%
	\widetilde{\mspace{-2mu}\rule{0pt}{1.3ex}\smash[t]{\phi}}%
}

\def\0{{\bm 0}}

\def\B{{\bm B}}
\def\C{{\bm C}}
\def\D{{\bm D}}

\def\H{{\bm H}}
\def\I{{\bm I}}

\def\K{{\bm K}}

\def\M{{\bm M}}
\def\N{{\bm N}}
\def\P{{\bm P}}
\def\Q{{\bm Q}}
\def\R{{\bm R}}

\def\U{{\bm U}}
\def\V{{\bm V}}

\def\X{{\bm X}}
\def\Y{{\bm Y}}
\def\Z{{\bm Z}}

\def\BSigma{\boldsymbol{\Sigma}}

\def\BPhi{\boldsymbol{\Phi}}

\def\E{{\mathbb{E}}}
\def\Ssigma{{\mathbf{\Sigma}}}

\def\RR{\mathbb{R}}
\def\Rbb{\mathbb{R}}

\def\SS{\mathbb{S}}

\newenvironment{proofof}[1]{\noindent{\it Proof of #1}.}{\hfill$\qed$\par}

\makeatletter
\renewcommand{\@biblabel}[1]{[#1]\hfill}
\makeatother
\renewcommand{\citet}[1]{\citep{#1}}

\begin{document}
\footnotetext[1]{Equal contribution.}

\maketitle

\begin{abstract}
We propose efficient random features for approximating a new and rich class of kernel functions that we refer to as \emph{Generalized Zonal Kernels (GZK)}.
Our proposed GZK family, generalizes the zonal kernels (i.e., dot-product kernels on the unit sphere) by introducing \emph{radial factors} in their Gegenbauer series expansion, and includes a wide range of ubiquitous kernel functions such as the entirety of dot-product kernels as well as the Gaussian and the recently introduced Neural Tangent kernels.
Interestingly, by exploiting the reproducing property of the Gegenbauer polynomials, we can construct efficient random features for the GZK family based on randomly oriented Gegenbauer kernels.
We prove subspace embedding guarantees for our Gegenbauer features which ensures that our features can be used for approximately solving learning problems such as kernel k-means clustering, kernel ridge regression, etc. 
Empirical results show that our proposed features outperform recent kernel approximation methods.
\end{abstract}

\section{Introduction} \label{sec-intro}

Kernel methods are undoubtedly an important family of learning algorithms, which are applicable
for a wide range of tasks, e.g. regression~\cite{saunders1998ridge}, clustering~\cite{dhillon2004kernel}, graph learning~\cite{vishwanathan2010graph}, non-parametric modeling~\cite{rasmussen2004gaussian} as well as wide deep neural networks analysis~\cite{jacot2018neural,lee2019wide}.
However, unfortunately, they tend to suffer from scalability issues, often due to the fact that applying the aforementioned methods requires operating on the kernel matrix (Gram matrix) of the data, whose size scales quadratically in the number of training samples. For example, solving kernel ridge regression generally requires a prohibitively large quadratic memory and a runtime that is in the order of matrix inversion.
To alleviate this issue, there has been a long line of efforts on efficiently approximating kernel matrices by low-rank factors~\cite{williams2001using,rahimi2007random,avron2014subspace,alaoui2014fast,musco2017recursive,avron2017random,zandieh2020scaling,ahle2020oblivious, woodruff2020near}. Most relevant to this work is the so-called {\em random features} approach, originally proposed by~\citet{rahimi2007random}. 

In this work, we propose efficient random features for approximating a new and rich class of kernel functions that we refer to as \emph{Generalized Zonal Kernels (GZK)} (see \cref{def-generalized-zonal-kernels}).
Our proposed class of kernels extends the zonal kernels (i.e., dot-product kernels restricted to the unit sphere) to entire $\RR^d$ space, and includes a wide range of ubiquitous kernels (e.g. the entire family of dot-product kernels and the Gaussian kernel), and the recently introduced Neural Tangent kernels \cite{jacot2018neural}. We start by considering the series expansion of zonal functions in terms of the Gegenbauer polynomials, which are central in our analysis. Then we generalize these kernels by allowing \emph{radial factors} in the Gegenbauer expansion. We construct the GZK family of kernels in \cref{sec-generalized-zonal-kernels}. We design efficient random features for this class of kernels by exploiting various properties of Gegenbauer polynomials and using leverage scores sampling techniques~\cite{li2013iterative}.

Specifically, for a given GZK function and its corresponding kernel matrix $\K \in \RR^{n \times n}$, we seeks to find a low-rank matrix that can serve as a proxy to the kernel matrix $\K$. We present an algorithm that for given $\varepsilon, \lambda > 0$, computes a matrix $\Z \in \RR^{m \times n}$ such that $\Z^\top \Z$ is an \emph{$(\varepsilon, \lambda)$-spectral approximation} to the GZK kernel matrix $\K$, meaning that %
\begin{align} \label{eq-spec-approx-truncated-features}
    \frac{\K + \lambda \I}{1+\varepsilon} \preceq \Z^\top \Z + \lambda \I  \preceq \frac{\K + \lambda \I}{1-\varepsilon}.
\end{align}
The spectral approximation guarantee can be directly used to obtain statistical and algorithmic guarantees for downstream kernel-based learning applications, such as bounds on the empirical risk of kernel ridge regression~\cite{avron2017random}.

\subsection{Overview of Our Contributions}

In this work, we define a rich class of kernels based on Gegenbauer polynomials, which are a class of orthogonal polynomials that include Chebyshev and Legendre polynomials and are widely employed in approximation theory~\cite{gautschi2004orthogonal}. We then present efficient random features for this new family of kernels by using the fact that Gegenbauer kernels induce a natural feature map on themselves because of their reproducing property (see \cref{lem-gegen-kernel-properties} for details).
To the best of our knowledge, this is the first work on random features of orthogonal polynomials with provable guarantees. We analyze our proposed random features and prove that they spectrally approximate the exact kernel matrix. 
Our contributions are listed as follows,

\begin{itemize}[wide, labelwidth=!, labelindent=1pt]
	\item We extend the zonal kernels from unit sphere to entire $\RR^d$ by adding radial components to the Gegenbauer series expansion of such kernels in \cref{def-generalized-zonal-kernels}. Then we propose the Mercer decomposition of this class of kernels based on Gegenbauer polynomials in \cref{lem-gzk-feature-map}.
	
	\item We show that our newly proposed class of kernels is rich and contains all dot-product kernels~\cref{lem-dot-prod-spherical-harmonic-expansion}, as well as Gaussian and Neural Tangent kernels~\cref{appndx-ntk-is-gzk}.
	
	\item We propose efficient random feature for our proposed class of kernels in \cref{def-random-feature-construction} and prove both spectral approximation and projection-cost preserving guarantees for our proposed features in \cref{thm:main-spectral-approx} and \cref{thm:main-proj-cost-preserv-approx}. These properties ensure that our random features can be used for downstream learning tasks such as kernel regression, kernel $k$-means, and principal/canonical component analysis, see \cref{sec-learning-tasks}.
	
	\item We apply our main spectral approximation results on dot-product and Gaussian kernels and show our method gives improved random features for these types of kernels in \cref{thm-spectral-truncated-random-features} and \cref{thm-spectral-approx-Gaussian}.
    \item Our empirical results verify that the proposed method outperforms previous approaches for approximating the Gaussian kernel.
\end{itemize}

\subsection{Related Work}

A popular line of work on kernel approximation is based on the random Fourier features method~\cite{rahimi2007random}, which works well for shift-invariant kernels and with some modifications can embed the Gaussian
kernel near optimally in constant dimension~\cite{avron2017random}. 
Other random feature constructions have been suggested for a variety of kernels, e.g., arc-cosine kernels~\cite{cho2009kernel}, polynomial kernels~\cite{pennington2015spherical}, and Neural Tangent kernels~\cite{zandieh2021scaling}.

For the polynomial kernel, sketching methods have been developed extensively~\cite{avron2014subspace, pham2013fast, woodruff2020near, song2021fast}. For example, \citet{ahle2020oblivious} proposed a subspace embedding for high-degree Polynomial kernels as well as the Gaussian kernel. However, approximating non-polynomial kernels using these tools require sketching the Taylor expansion of the kernel which can perform somewhat poorly due to slow convergence rate of Taylor series. On the other hand, we focus on Gegenbauer series that generally converge faster~\cite{fox1968chebyshev,mason2002chebyshev}.

Another popular kernel approximation approach is the Nystr\"om method~\cite{williams2001using, yang2012nystrom}. While the recursive Nystr\"om sampling of \citet{musco2017recursive} can embed kernel matrices using near optimal number of landmarks, this method is inherently data dependent, so unlike our data oblivious random features, it cannot provide one-round distributed protocols and/or single-pass streaming algorithms.


\section{Preliminaries} \label{sec-prelim}

\paragraph{Notations.} 
We denote by $\SS^{d-1}$ the unit sphere in $d$ dimension.
We use $|\SS^{d-1}|=\frac{2\pi^{d/2}}{\Gamma(d/2)}$ to denote the surface area of the unit sphere $\SS^{d-1}$ and $\mathcal{U}(\SS^{d-1})$ to denote the uniform probability distribution on $\SS^{d-1}$.
We use $\mathbbm{1}_{\{\mathcal{E}\}}$ as an indicator function for event $\mathcal{E}$. All matrices are in boldface, e.g., $\K$, and we let $\I_n$ be the $n \times n$ identity matrix and sometimes omit the subscript. 
For any function $\kappa(\cdot)$ and any integer $i$ we denote the $i^{th}$ derivative of $\kappa$ with $\kappa^{(i)}(t)$ or $\frac{d^{i}}{dt^{i}} \kappa(t)$.
We use $\|\cdot \|$ and $\|\cdot\|_{\op}$ to denote the $\ell_2$-norm of vectors and the operator norm of matrices, respectively. The statistical dimension of a positive semidefinite matrix $\K$ and parameter $\lambda\geq0$ is defined as $s_\lambda\coloneqq\trace{\left(\K(\K+\lambda \I)^{-1}\right)}$.

\subsection{Gegenbauer Polynomials}
The Gegenbauer polynomial (a.k.a. \emph{ultraspherical polynomial}) of degree $\ell\ge 0$ in dimension $d\ge2$ is given by
\begin{align}
	P_d^\ell(t) \coloneqq \sum_{j=0}^{\lfloor \ell/2 \rfloor} c_j \cdot t^{\ell-2j} \cdot (1 - t^2)^j,
\end{align}
where $c_0 = 1$ and $c_{j+1} = - \frac{(\ell - 2j)(\ell - 2j - 1)}{2(j+1)(d-1 + 2j)} c_j$ for $j = 0,1, \ldots \lfloor \ell/2 \rfloor-1$. 
This class of polynomials includes Chebyshev polynomials of the first kind when $d=2$ and Legendre polynomials when $d=3$. Furthermore, when $d=\infty$, these polynomials reduce to monomials i.e., $P_{\infty}^\ell(t) = t^\ell$. They also fall into the important class of Jacobi polynomials.

Gegenbauer polynomials satisfy an orthogonality property on interval $[-1,1]$ with respect to measure $(1-t^2)^{\frac{d-3}{2}}$:
\begin{align} \label{eq-gegen-orthogonality}
	\int_{-1}^1 P_d^\ell(t) P_d^{\ell'}(t) (1-t^2)^{\frac{d-3}{2}} \, dt = \frac{\left| \SS^{d-1} \right| \cdot \mathbbm{1}_{\{\ell = \ell'\}}}{\alpha_{\ell,d} \cdot \left| \SS^{d-2} \right|},
\end{align}
where $\alpha_{\ell,d}$ is the dimensionality of the space of \emph{spherical harmonics} of order $\ell$ in dimension $d$ defined as $\alpha_{0,d} \coloneqq1, \alpha_{1,d}\coloneqq d$ and for $\ell \ge 2$
\begin{align}\label{eq:dim-spherical-harmonics}
	\alpha_{\ell,d} \coloneqq {d+\ell-1 \choose \ell} - {d+\ell-3 \choose \ell-2}.
\end{align}
The following alternative expression for $P_d^\ell(t)$, proved in \cite{morimoto1998analytic}, is known as Rodrigues' formula,
\begin{align}\label{eq:Gegenbauer-poly-derivate-def}
    P_d^\ell(t) = \frac{ (-1)^\ell \Gamma\left( \frac{d-1}{2}\right)}{2^\ell (1-t^2)^{\frac{d-3}{2}} \Gamma\left( \ell + \frac{d-1}{2}\right) }  \frac{d^\ell \left( 1-t^2 \right)^{\ell + \frac{d-3}{2}}}{dt^\ell}
\end{align}
for any $d \geq 3$.

\subsection{Hilbert Space of Function in $L^2\left( \SS^{d-1} , \RR^s \right)$}
For any integer $s \ge 1$ and any vector-valued functions $f, g \in L^2(\SS^{d-1} , \RR^s)$ meaning that $f,g:\SS^{d-1} \to \RR^s$, we define the inner product of these maps as follows,
\begin{equation}\label{eq:def-inner-prod-multimaps-unit-sphere}
    \langle f, g \rangle_{L^2( \SS^{d-1} , \RR^s )} \coloneqq \mathbb{E}_{w\sim \mathcal{U}(\SS^{d-1})} \left[\langle f(w), g(w) \rangle \right].
\end{equation}
With this inner product, $L^2\left( \SS^{d-1} , \RR^s \right)$ is a \emph{Hilbert space}, with norm $\|f\|_{L^2( \SS^{d-1} , \RR^s )} = \sqrt{\langle f, f \rangle_{L^2( \SS^{d-1} , \RR^s )}}$. Furthermore, we shorten the notation for the space of square-integrable functions $L^2(\SS^{d-1}, \RR)$ to $L^2(\SS^{d-1})$.

\subsection{Gegenbauer Polynomials as Kernel Functions}
The Gegenbauer polynomials naturally provide positive definite dot-product kernels on the unit sphere $\SS^{d-1}$.  In fact, \citet{schoenberg1988positive} proved that a dot-product kernel $k(x,y) = \kappa(\inner{x,y})$ is positive definite if and only if $\kappa(t) = \sum_{\ell=0}^\infty c_\ell P_d^{\ell}(t)$
with all $c_\ell \geq 0$ (see Theorem 3 therein). 

Particularly the following reproducing property of Gegenbauer polynomials is useful which follows from the Funk–Hecke formula (See~\cite{atkinson2012spherical}).
\begin{lemma}[Reproducing property of Gegenbauer kernels] 
\label{lem-gegen-kernel-properties}
	Let $P_d^\ell(\cdot)$ be the Gengenbauer polynomial of degree $\ell$ in dimension $d$. For any $x,y \in \SS^{d-1}$:
	\begin{align*}
		P_d^\ell(\langle x, y \rangle) = \alpha_{\ell,d} \cdot \mathbb{E}_{w\sim \mathcal{U}(\SS^{d-1})} \left[ P_d^\ell\left( \langle x , w \rangle \right) P_d^\ell\left( \langle y , w \rangle \right) \right] ,
	\end{align*}
	Furthermore, for any $\ell' \neq \ell$: 
	\begin{align*}
		\mathbb{E}_{w\sim \mathcal{U}(\SS^{d-1})} \left[ P_d^\ell\left( \langle x , w \rangle \right) \cdot P_d^{\ell'}\left( \langle y , w \rangle \right) \right] = 0.
	\end{align*}
\end{lemma}

\section{Generalized Zonal Kernels (GZK)}

In this section, we introduce our proposed class of Generalized Zonal Kernels (GZK). We start by proposing a practical Mercer decomposition of zonal kernels, i.e., dot-product kernels on the unit sphere, and then extend it to a large class of kernel functions -- Generalized Zonal Kernels. 

\subsection{Warm-up: Mercer Decomposition of Zonal Kernels}

A function $k:\SS^{d-1} \times \SS^{d-1} \rightarrow \RR$ is called {\it zonal kernel} if it can be represented by $k(x,y) = \kappa(\inner{x,y})$ for some scalar function $\kappa:[-1,1] \rightarrow \RR$.  Note that zonal kernels are 
rotation invariant, i.e., $k(x,y) = k(\R x, \R y)$ for any rotation matrix $\R \in \RR^{d\times d}$. Due to this property, zonal kernels have been used in various geo-science applications including climate change simulation~\cite{sanderson2010climate}, Ozone prediction~\cite{su2020prediction} and mantle convection~\cite{bercovici2003generation}. 

Using 
Gegenbauer series expansion $\kappa(t) = \sum_{\ell=0}^\infty c_\ell P_d^{\ell}(t)$, we have
\begin{align} \label{eq-gegen-expansion}
	k(x,y) = \kappa(\inner{x,y})=\sum_{\ell=0}^\infty c_\ell \cdot P_d^{\ell}(\langle x, y \rangle).
\end{align}
By orthogonality property in~\cref{eq-gegen-orthogonality}, $c_\ell$ can be computed as
\begin{align} \label{eq-gegen-expansion-coeffs}
	c_\ell = \alpha_{\ell, d} \cdot \frac{|\SS^{d-2}|}{|\SS^{d-1}|} \cdot \int_{-1}^1 \kappa(t) P_d^{\ell}(t) (1-t^2)^{\frac{d-3}{2}} dt.
\end{align}
It is known that polynomial approximation with Chebyshev series (i.e., $d=2$) generally has faster convergence rate compared to Taylor series (i.e., $d=\infty$)~\cite{fox1968chebyshev,mason2002chebyshev}. We empirically verify that the Gegenbauer series (i.e., $2 < d < \infty$) interpolates between Taylor and Chebyshev series in \cref{sec-exp-gegen-poly-approx}. 

Throughout this work, we assume that $\kappa(\cdot)$ is an analytic function so that the corresponding Gegenbauer series expansion exists and converges.
With \cref{eq-gegen-expansion} in-hand and applying \cref{lem-gegen-kernel-properties} we obtain a Mercer decomposition of zonal kernels.

\begin{restatable}[Feature map for zonal kernels]{lemma}{zonalkernelfeaturemap} \label{lem-zonal-kernel-feature-map}
	Suppose $\kappa:[-1,1] \rightarrow \RR$ is analytic and let $\{ c_\ell \}_{\ell=0}^\infty$ be the coefficients of its Gegenbauer series expansion in dimension $d \geq 2$. For $x,w\in \SS^{d-1}$, define the real-valued function $\phi_x \in L^2(\SS^{d-1})$ as
	\begin{align} \label{eq-zonal-feature-map}
		\phi_x(w) \coloneqq \sum_{\ell=0}^\infty \sqrt{c_{\ell} \cdot \alpha_{\ell,d}} \cdot P_d^{\ell}(\inner{x,w}).
	\end{align}
	Then, for all $x,y\in \SS^{d-1}$, it holds that
	\begin{align}
		\mathbb{E}_{w\sim \mathcal{U}(\SS^{d-1})} \left[\phi_x(w)\cdot \phi_y(w) \right] = \kappa(\inner{x,y}).
	\end{align}
\end{restatable}

The proof of \cref{lem-zonal-kernel-feature-map} can be found in \cref{sec-proof-zonal-kernel-feature-map}.

\subsection{Extension to Dot-product Kernels and Beyond} \label{sec-generalized-zonal-kernels}
In this section, we generalize the zonal kernel functions from $\SS^{d-1}$ to entire $\RR^{d}$ 
by factorizing the kernel function into angular and radial parts.

\begin{defn}[Generalized zonal kernels] \label{def-generalized-zonal-kernels}	
	For an integer $s \ge1$ and a sequence of vector-valued functions $h_{\ell}: \RR \to \RR^s$ for $\ell = 0,1,\ldots$, we define the \emph{generalized zonal kernel (GZK)} of order $s$ as 
	\begin{align} \label{eq:dot-prod-kernel-Gegenbauer-expansion}
	\displaystyle k(x,y) \coloneqq \sum_{\ell=0}^\infty \langle h_{\ell}(\|x\|) , h_{\ell}(\|y\|) \rangle P_d^{\ell}\left( \frac{\langle x, y \rangle}{\|x\|\|y\|}\right).
	\end{align}
\end{defn}
We remark that for any series of real-valued vector functions $h_{\ell}: \RR \to \RR^s$, \cref{eq:dot-prod-kernel-Gegenbauer-expansion} defines a valid positive definite kernel (we give the Mercer decomposition of the GZK function in \cref{lem-dot-prod-spherical-harmonic-expansion}). While we defined the GZK functions for finite order $s$, the definition can be extended to include $s=+\infty$ by letting $h_\ell(t)$ be a map to the square-summable sequences (a.k.a. $l^2$-sequence-space\footnote{$l^2$ space not be confused with the index $\ell$ in functions $h_\ell(\cdot)$}) and letting the term $\langle h_{\ell}(\|x\|) , h_{\ell}(\|y\|) \rangle$ in \cref{eq:dot-prod-kernel-Gegenbauer-expansion} be the standard $l^2$-inner-product of sequences $h_{\ell}(\|x\|), h_{\ell}(\|y\|)$.

The class of GZK in \cref{def-generalized-zonal-kernels} includes a wide range of familiar kernel functions such as all dot-product kernels, the Gaussian and Neural Tangent Kernels. In the following lemma we show that dot-products kernels are GZK.
\begin{restatable}[Dot-product kernels are GZKs]{lemma}{lmmdotprodaregzk} \label{lem-dot-prod-spherical-harmonic-expansion}
	For any $x, y \in \RR^d$, any integer $d \ge 3$, and any dot-product kernel $k(x,y) = \kappa(\langle x, y \rangle)$ with analytic $\kappa(\cdot)$, the eigenfunction expansion of $k(x,y)$ can be written as,
	\begin{align*}
		k(x,y) \coloneqq \sum_{\ell=0}^\infty \left( \sum_{i=0}^\infty \widetilde{h}_{\ell,i}(\|x\|) \widetilde{h}_{\ell,i}(\|y\|) \right) P_d^{\ell}\left( \frac{\langle x, y \rangle}{\|x\| \|y\|}\right),
	\end{align*}
	where $\widetilde{h}_{\ell,i}(\cdot)$ are real-valued monomials defined as follows for integers $\ell,i \ge 0$ and any $t \in \RR$:
	\begin{align}\label{eq:gegenbauer-expansion-coeff-computation}
	    \widetilde{h}_{\ell,i}(t) \coloneqq \sqrt{ \frac{\alpha_{\ell,d}}{2^\ell}  \frac{\Gamma(\frac{d}{2})~~\kappa^{(\ell+2i)}(0)}{\sqrt{\pi}(2i)!}   \frac{\Gamma(i+\frac{1}{2}) }{\Gamma(i+\ell+ \frac{ d}{2})} } \cdot t^{\ell+2i}.
	\end{align}
\end{restatable}
The proof of \cref{lem-dot-prod-spherical-harmonic-expansion} is provided in \cref{sec-proof-lem-dot-prod-spherical-harmonic-expansion}. 
The proof starts by expressing the monomials in Taylor series expansion of $\kappa(\inner{x,y})$ in the Gegenbauer basis, i.e., $\inner{x,y}^j = (\|x\|\|y\|)^j \cdot \langle \frac{x}{\|x\|}, \frac{y}{\|y\|} \rangle^j= (\|x\|\|y\|)^j \cdot \sum_{\ell=0}^j c_\ell P_d^\ell(\frac{\inner{x,y}}{\|x\|\|y\|})$. The coefficients $c_\ell$ can be computed using \cref{eq-gegen-expansion-coeffs} along with the Rodrigues' formula in~\cref{eq:Gegenbauer-poly-derivate-def}.
\cref{lem-dot-prod-spherical-harmonic-expansion} shows that any dot-product kernel $\kappa(\cdot)$ is indeed a GZK of order $s$ if its derivatives $\kappa^{(2i)}(t)$ at $t=0$ for $i \geq s$ are zeros.
If the derivatives of $\kappa(t)$ do not vanish at $t=0$ then the kernel can be a GZK of potentially infinite order $s=+\infty$ with $h_\ell(t) = [\widetilde{h}_{\ell,i}(t)]_{i=0}^\infty$, where $\widetilde{h}_{\ell,i}(\cdot)$ are defined as per \cref{eq:gegenbauer-expansion-coeff-computation}. 
In \cref{sec:gauss-ntk-spectral-approx} we show that $\widetilde{h}_{\ell,i}(\cdot)$ rapidly decay with respect to $i$ thus dot-product kernels can be tightly approximated by GZKs with small finite order $s$.
Furthermore, when inputs are on the unit sphere, i.e., $\norm{x} = 1$, the radial functions $\widetilde{h}_{\ell,i}(\|x\|)$ turn out to be constant so a dot-product kernel on the sphere (a.k.a zonal kernel as per \cref{eq-gegen-expansion}) is a GZK of order $s=1$.

Now we present a feature map for the GZK which will be the basis of our efficient random features.

\begin{restatable}[Feature map for GZK]{lemma}{featuremapgzk}\label{lem-gzk-feature-map}
    Consider a GZK $k(\cdot, \cdot)$ with real-valued functions $h_{\ell}: \RR \to \RR^s$ for $\ell=0,1,\dots$ as in \cref{def-generalized-zonal-kernels}.
	For any $x \in \RR^{d},w \in \SS^{d-1}$, define the function $\phi_x \in L^2( \SS^{d-1} , \RR^s)$ as 
	\begin{align} \label{eq-featuremap-gen-zonal-kernls}
		\phi_x (w) \coloneqq \sum_{\ell=0}^\infty \sqrt{\alpha_{\ell,d}} ~ h_{\ell}(\|x\|) ~ P_d^{\ell}\left( \frac{\langle x, w \rangle}{\|x\|}
		\right).
	\end{align}
    Then, for any $x,y \in \RR^d$, it holds that
	\begin{align*}
		\mathbb{E}_{w \sim \mathcal{U}(\SS^{d-1})} \left[ \inner{\phi_x(w), \phi_y(w)} \right] = k(x, y).
	\end{align*}
\end{restatable}

The proof of \cref{lem-gzk-feature-map} is given in 
\cref{sec-proof-lem-gzk-feature-map}.  For this feature map to be well-defined we require the series in \cref{eq-featuremap-gen-zonal-kernls} to be convergent for every $x \in \RR^d$ in our dataset.

{\bf Remark.}
Several works have attempted to extend simple zonal Kernels from $\SS^{d-1}$ to $\RR^{d}$~\cite{smola2001regularization,cho2010large, scetbon2021spectral}.
They focus on the eigensystem of the dot-product kernels based on the spherical harmonics. 
However, it is intractable to compute spherical harmonics in general~\cite{minh2006mercer} which renders the above-mentioned eigendecomposition results mainly existential and non-practical.
On the other hand, we propose a computationally practical Mercer decomposition of the GZK (and a fortiori dot-product kernels) in \cref{lem-gzk-feature-map}, which unlike \cite{smola2001regularization} does not rely on spherical harmonics and will lead to efficient kernel approximations.

\def\wtK{\widetilde{\K}}

\section{Spectral Approximation of GZK}

In this section, we propose random features of GZK using our feature map in \cref{eq-featuremap-gen-zonal-kernls} and analyze their approximation guarantee as in~\cref{eq-spec-approx-truncated-features}.  We first introduce the following notations that are essential in our analysis.

Consider a dataset $\X=[x_1,, \ldots, x_n]\in \RR^{d \times n}$ and a GZK $k(\cdot, \cdot)$ as per \cref{def-generalized-zonal-kernels} and let the $n$-by-$n$ kernel matrix $\K$ be defined as $[\K]_{i,j} \coloneqq k(x_i, x_j)$.
Let $\phi_{x_j}$ be the feature map defined in \cref{eq-featuremap-gen-zonal-kernls} for all $j\in[n]$. For $v \in \RR^n$, we define an operator $\BPhi : \RR^{n} \to L^2\left( \SS^{d-1} , \RR^s \right)$ (a.k.a. quasi-matrix) as follows,
\begin{align}\label{eq:def-feature-operator}
    \BPhi \cdot v \coloneqq \sum_{j=1}^n v_j \cdot \phi_{x_j}.
\end{align}
The adjoint of this operator $\BPhi^* : L^2\left( \SS^{d-1} , \RR^s \right) \to \RR^n$ is the following for $f \in L^2\left( \SS^{d-1} , \RR^s \right)$ and $j \in [n]$,
\begin{align}
    [\BPhi^* f]_j = \langle \phi_{x_j}, f \rangle_{L^2(\SS^{d-1}, \RR^s)},
\end{align}
where the inner product above is defined as per \cref{eq:def-inner-prod-multimaps-unit-sphere}.
With this definition, it follows from \cref{lem-gzk-feature-map} that
\[ \BPhi^* \BPhi = \K. \]
Our approach for spectrally approximating $\K$ is sampling the ``rows'' of the quasi-matrix $\BPhi$ with probabilities proportional to their ridge leverage scores~\cite{li2013iterative}.
The ridge leverage scores of $\BPhi$ are defined as follows,
\begin{defn}[Ridge leverage scores of $\BPhi$]\label{defn:leverage-score}
Let $\BPhi : \RR^n \to L^2(\SS^{d-1}, \RR^s )$ be the operator defined in \cref{eq:def-feature-operator}. Also, for every $w \in \SS^{d-1}$, define $\Phi_w \in \RR^{n \times s}$ as,
\begin{align}\label{eq:row-feature-operator}
    \Phi_w \coloneqq \left[ \phi_{x_1}(w), \phi_{x_2}(w), \ldots \phi_{x_n}(w) \right]^\top.
\end{align}
For any $\lambda > 0$, the row leverage scores of $\BPhi$ are defined as,
\begin{align}
    \tau_\lambda(w) \coloneqq \trace \left( \Phi_w^\top \cdot \left( \K + \lambda \I \right)^{-1} \cdot \Phi_w \right).
\end{align}
\end{defn}

An important quantity for the spectral approximation to $\K$
is the \emph{average} of the ridge leverage scores with respect to the uniform distribution on $\SS^{d-1}$
which is equals to \emph{statistical dimension}:
\begin{align}
\E_{w \sim \mathcal{U}(\SS^{d-1})} [\tau_\lambda(w) ] = \trace ( \K( \K + \lambda \I)^{-1}) = s_\lambda.
\end{align}

{\bf Remark.}
Our definition of leverage scores is slightly non-standard and different from the prior works such as \cite{avron2017faster,avron2019universal} because it is not normalized with the distribution of $w \sim \mathcal{U}(\SS^{d-1})$. The difference stems from the definition of inner product in $L^2(\SS^{d-1}, \RR^s)$ space in~\cref{eq:def-inner-prod-multimaps-unit-sphere}.

\subsection{Random Features Based on the Leverage Scores}
\label{sec-spec-truncated-features}
In this section, we propose our random features based on sampling according to the leverage scores of $\BPhi$, and show that it is able to spectrally approximate $\K$.
However, computing the leverage scores exactly is expensive in general and even if we could it is not necessarily easy to sample from them efficiently. 
So, we focus on approximating the leverage scores of the GZK with a distribution which is easy to sample from. 
Specifically, we find a $\widehat{\tau}_\lambda(\cdot)$ such that $\widehat{\tau}_\lambda(w) \ge \tau_\lambda(w)$ for all $w \in \SS^{d-1}$. 
For any GZK and its corresponding feature operator defined in \cref{eq:def-feature-operator}, we have the following upper bound,
\begin{restatable}[Upper bound on leverage scores of GZK]{lemma}{leveragescoreupperbound} \label{lem:leverage-score-upper-bound}
For any dataset $\X = [x_1, x_2, \ldots, x_n] \in \RR^{d \times n}$, let $\BPhi$ be the feature operator for the order $s$ GZK on $\X$ defined in \cref{eq:def-feature-operator}. For any $\lambda>0$ and $w \in \SS^{d-1}$, the ridge leverage scores of $\BPhi$ defined in~\cref{defn:leverage-score} are uniformly upper bounded by
\[ 
\tau_\lambda(w) \le \sum_{\ell=0}^\infty \alpha_{\ell,d} \min \left\{ \frac{\pi^2 (\ell+1)^2}{6\lambda} \sum_{j \in [n]} \left\| h_{\ell}(\|x_j\|) \right\|^2 , s \right\}. \]
\end{restatable}

{\it Proof Sketch.} 
To find a proper upper bound on the ridge leverage function, we first show that it can be expressed as the sum of a collection of regularized least-squares problems, i.e., $\tau_{\lambda} = \sum_{i=1}^s \tau_i^*$ for
\[
\tau_i^* \coloneqq \min_{g_i \in L^2(\SS^{d-1},\RR^s)} \norm{g_i}_{L^2(\SS^{d-1},\RR^s)}^2 + \lambda^{-1} \left\| \BPhi^* g_i - \Phi_w^i \right\|_2^2,
\]
where $\Phi_w^i \in \mathbb{R}^n$ is the $i^{th}$ column of matrix $\Phi_w$ defined in \cref{eq:row-feature-operator}. 
Intuitively, the function $g_i(\sigma) = \sum_{\ell=q}^\infty \alpha_{\ell,d} \cdot P_d^{\ell}\left( \langle \sigma,w \rangle\right) \cdot e_i$, where $e_i$ is the $i^{th}$ standard basis vector in $\RR^s$, can zero out the second term in the above objective function, by \cref{lem-gegen-kernel-properties}, while making the first term infinite $\norm{g_i}_{L^2(\SS^{d-1},\RR^s)}^2 = \sum_{\ell=q}^\infty \alpha_{\ell,d}$. On the other hand, for $g_i = 0$, the first term in the objective function will be zero while the second term will be as large as $\lambda^{-1} \left\| \Phi_w^i \right\|_2^2$. 

To find a balance between these two extremes, we make the heavy radial components in the second term small, i.e., $\ell$'s such that $\left\| h_{\ell}(\|x_j\|) \right\|^2$ is large, and ignore the small components to keep the norm of $g_i$ as small as possible. Specifically, we choose the following feasible solution that is nearly optimal for the above least-squares problem 
\[
\widehat{g}_i(\sigma) = \left(\sum_{\ell=q}^\infty \alpha_{\ell,d} \mathbbm{1}_{ \left\{ \sum_{j} \left\| h_{\ell}(\|x_j\|) \right\|^2 \ge \lambda s \right\} } P_d^{\ell}\left( \langle \sigma,w \rangle\right)\right) \cdot e_i.
\]

Plugging this to the minimization problem gives the lemma.
The full proof is in \cref{appndx;leverage-score-upperbound}.\hfill \qed

We will show in \cref{sec:gauss-ntk-spectral-approx} that the bound in \cref{lem:leverage-score-upper-bound} is typically small for all practically important kernels because the radial components $h_\ell(\cdot)$ rapidly decay as $\ell$ increases.
Inspired by this uniform bound on leverage score, we propose the following random features for the GZK by uniformly sampling the rows of the feature operator $\BPhi$ in \cref{eq:def-feature-operator}.

\begin{defn}[Random features for Generalized Zonal Kernels]\label{def-random-feature-construction}
For any GZK as per \cref{def-generalized-zonal-kernels} and dataset $\X \in \RR^{d \times n}$, sample $m$ i.i.d. points $w_1, \ldots ,w_m \sim \mathcal{U}(\SS^{d-1})$ and let $\Phi_{w_1},  \ldots, \Phi_{w_m} \in \RR^{n \times s}$ be defined as per \cref{eq:row-feature-operator}, then define the features matrix $\Z \in \RR^{(m \cdot s) \times n}$ as:
\begin{align} \label{eq:zonal-kernel-random-features}
    \Z \coloneqq \frac{1}{\sqrt{m}} \cdot \left[ \Phi_{w_1}, \ldots ,\Phi_{w_m} \right]^\top.
\end{align}
\end{defn}
These random features are \emph{unbiased}, i.e., $\E\left[ \Z^\top \Z \right] = \K$.

\subsection{Main Theorems} \label{sec-main-theorems}

We now formally prove that for the class of GZKs, the random features in \cref{def-random-feature-construction}  yield a spectral approximation to the kernel matrix $\K$ with enough number of features. 

\begin{restatable}[Spectral approximation of GZK]{theorem}{specapproxgeneralizedzonal}\label{thm:main-spectral-approx}
For any dataset $\X = [x_1, x_2, \ldots x_n] \in \RR^{d \times n}$, let $\K$ be the corresponding GZK kernel matrix (\cref{def-generalized-zonal-kernels}). For any $0 < \lambda \le \norm{\K}_{\op}$, let $\Z \in \RR^{(m\cdot s) \times n}$ be the random features matrix defined in \cref{def-random-feature-construction}. Also let $s_\lambda$ be the statistical dimension of $\K$. For any $\varepsilon , \delta >0$, if  $m \ge \frac{8}{3\varepsilon^{2}} \log \frac{16s_\lambda}{\delta} \cdot \sum_{\ell=0}^\infty \alpha_{\ell,d} \min \left\{ \frac{\pi^2 (\ell+1)^2}{6\lambda} \sum_{j \in [n]} \left\| h_{\ell}(\|x_j\|) \right\|^2 , s \right\}$,
then with probability of at least $1-\delta$, \begin{equation} \label{eq:spectral-approx-thm}
\frac{\K + \lambda \I}{1 + \varepsilon}
\preceq
\Z^\top \Z  + \lambda \I
\preceq
\frac{\K + \lambda \I}{1 - \varepsilon}.\end{equation}
\end{restatable}
We provide the proof of \cref{thm:main-spectral-approx} in \cref{sec:proof-thm-spectral}.
The proof follows the standard approach studied in~\cite{avron2017random}.
By \cref{lem:leverage-score-upper-bound}, there exists a bound $U \ge \tau_\lambda(w)$ for all $w \in \SS^{d-1}$. This gives upper bounds of both the operator norm and the second moment of our kernel estimator.  Applying a matrix concentration inequality~(e.g., Corollary 7.3.3 in \citet{tropp2015introduction}) with those bounds gives the result.

In addition to the basic spectral approximation guarantee of \cref{thm:main-spectral-approx}, we also prove that our random features method is able to produce \emph{projection-cost preserving} samples.

\begin{restatable}[Projection cost preserving GZK approximation]{theorem}{projcostpreservapproxgeneralizedzonal}\label{thm:main-proj-cost-preserv-approx}
Let $\K$ be the GZK kernel matrix as in \cref{thm:main-spectral-approx} with eigenvalues $\lambda_1 \ge \ldots \ge\lambda_n$. For any positive integer $r$, let $\lambda \coloneqq \frac{1}{r} \sum_{i=r+1}^n \lambda_i$ and let $s_\lambda$ be the statistical dimension of $\K$. For any $\varepsilon , \delta >0$, if $\Z \in \RR^{(m\cdot s) \times n}$ is the random features matrix defined in \cref{def-random-feature-construction} with $m \ge \frac{8}{3\varepsilon^{2}} \log \frac{16s_\lambda}{\delta} \cdot \sum_{\ell=0}^\infty \alpha_{\ell,d} \min \left\{ \frac{\pi^2 (\ell+1)^2}{6\lambda} \sum_{j \in [n]} \left\| h_{\ell}(\|x_j\|) \right\|^2 , s \right\}$, with probability at least $1-\delta$, the following holds for all rank-$r$ orthonormal projections $\P$:
\begin{align}
    (1 - \varepsilon)~\trace(\K - \P \K \P)
    \leq
    \trace( \Z^\top \Z - \P \Z^\top \Z \P)
    \leq
    (1 + \varepsilon)~\trace (\K - \P \K \P).
\end{align}
\end{restatable}
We prove \cref{thm:main-proj-cost-preserv-approx} in \cref{sec:proof-thm-project-cost}. This property ensures that it is possible to extract a near optimal low-rank approximation to the kernel matrix from our random features, thus they can be used for learning tasks such kernel $k$-means, principal component analysis (PCA) and Gaussian processes.  We provide how the projection-cost preserving cost can be applied to these tasks in \cref{sec-learning-tasks}.


\section{Application to Popular Kernels}\label{sec:gauss-ntk-spectral-approx}
So far we have showed GZKs can be spectrally approximated using the random features we designed in \cref{def-random-feature-construction}. We have also showed in \cref{lem-dot-prod-spherical-harmonic-expansion} and \cref{appndx-ntk-is-gzk} that all dot-product kernels as well as Gaussian and Neural Tangent Kernels are in the rich family of GZKs. 
Thus, our random features can be used to get a good spectral approximation for these kernels. 
In this section we answer the question of efficiency of our random features. 

Note that \cref{thm:main-spectral-approx} bounds the number of required features by $$\sum_{\ell=0}^\infty \alpha_{\ell,d} \min \left\{ \frac{\pi^2 (\ell+1)^2}{6\lambda} \sum_{j \in [n]} \left\| h_{\ell}(\|x_j\|) \right\|^2 , s \right\}.$$ We show that for dot-product and Gaussian kernels and datasets with bounded radius, the radial components $\sum_{j \in [n]} \left\| h_{\ell}(\|x_j\|) \right\|^2$ decay very fast as $\ell$ increases and effectively only the terms with degree $\ell \lesssim \log \frac{n}{\lambda}$ matter. This way, we get simple bounds on the number of required features for these kernels and also show that the features given in \cref{def-random-feature-construction} are efficiently computable.

\subsection{Dot-product Kernels}
We proved in \cref{lem-dot-prod-spherical-harmonic-expansion} that any dot-product kernel $k(x,y) = \kappa(\langle x, y \rangle)$ with analytic $\kappa(\cdot)$ is a GZK, thus can be spectrally approximated by \cref{thm:main-spectral-approx}.
To bound the number of required random features, we need to know how fast the monomials $\widetilde{h}_{\ell, i}(\cdot)$ in \cref{eq:gegenbauer-expansion-coeff-computation} decay as a function of $\ell$.
To bound the decay of $\widetilde{h}_{\ell,i}(\cdot)$, we first need to quantify the growth rate of the derivatives of $\kappa(\cdot)$. We assume that derivatives of $\kappa(\cdot)$ at zero can be characterized by the following exponential growth.

\begin{assp}\label{assumpt-growth-integral-kappa}
	For a dot-product kernel $\kappa(\cdot)$ suppose that there exist some constants $C_\kappa \ge 0$ and $\beta_\kappa \ge 1$ such that for any integer $\ell > d$,
	$\kappa^{(\ell)}(0) \le C_\kappa \cdot \beta_\kappa^\ell$.
\end{assp}
\citet{schoenberg1988positive} showed that for any dot-product kernel we have $\kappa^{(\ell)}(0) \ge 0$ for all $\ell$. 
\cref{assumpt-growth-integral-kappa} is commonly observed in popular kernel functions. For example, the exponential kernel $\kappa(\inner{x,y})=e^{\inner{x,y}}$ satisfies \cref{assumpt-growth-integral-kappa} with $C_\kappa=\beta_\kappa=1$. 

Now, for kernel $\kappa(\inner{x,y})$ and positive integers $s,q$ we let $k_{q,s}(x,y)$ be the order $s$ GZK as per \cref{def-generalized-zonal-kernels} whose corresponding radial functions $h_\ell: \RR \to \RR^s$ are defined as follows for $i \in [s]$ and $\ell \le q$
\begin{equation}\label{def-monomial-coeff-dot-prod-kernel}
    [h_{\ell}(t)]_i = \sqrt{ \frac{\alpha_{\ell,d}}{2^\ell}  \frac{\Gamma(\frac{d}{2})~\kappa^{(\ell+2i)}(0)}{\sqrt{\pi}(2i)!}  \frac{\Gamma(i+\frac{1}{2}) }{\Gamma(i+\ell+ \frac{ d}{2})} } \cdot t^{\ell+2i}
\end{equation}
and $h_\ell(t) \coloneqq 0$ for any $\ell > q$.
We show that under \cref{assumpt-growth-integral-kappa}, the GZK $k_{q,s}(x,y)$ tightly approximates $\kappa(\inner{x,y})$ for reasonably small values of $q$ and $s$, thus we can approximate $\kappa(\inner{x,y})$ by invoking \cref{thm:main-spectral-approx} on $k_{q,s}(x,y)$. Specifically, we prove the following theorem,

\begin{restatable}{theorem}{specapproxgeneralizedzonalfinal} \label{thm-spectral-truncated-random-features}
Suppose \cref{assumpt-growth-integral-kappa} holds for a dot-product kernel $\kappa(\inner{x,y})$. Given $\X=[x_1, \dots, x_n] \in \RR^{d \times n}$, assume that $\max_{j\in[n]}\norm{x_j} \le r$. Let $\K$ be the kernel matrix corresponding to $\kappa(\cdot)$ and $\X$.
For any $0 < \lambda \le \norm{\K}_{\op}$ and $\varepsilon , \delta >0$ let $s_\lambda$ be the statistical dimension of $\K$ and define $q= \max \left\{ d , 3.7 r^2\beta_\kappa, r^2\beta_\kappa + \frac{d}{2} \log \frac{3r^2 \beta_\kappa}{d} + \log \frac{C_\kappa n}{ \varepsilon \lambda} \right\}$.
There exists a randomized algorithm that can output $\Z \in \RR^{m \times n}$ with $m = \frac{5q^2}{4\varepsilon^{2}} \cdot {q + d-1 \choose q} \cdot \log \frac{16s_\lambda}{\delta}$, such that with probability at least $1-\delta$, $\Z^\top \Z$ is an $(\varepsilon,\lambda)$-spectral approximation to $\K$ as per \cref{eq-spec-approx-truncated-features}. Furthermore, $\Z$ can be computed in time $\bigo((m/q) \cdot \nnz{\X})$.
\end{restatable}
In \cref{sec-proof-thm-spectral-truncated-random-features} we provide more formal statement and proof.

\subsection{Gaussian Kernel}

\def\smallO{o}

The Gaussian kernel $g(x,y) = e^{-\|x-y\|_2^2/2}$ is a GZK as shown in \cref{eq-gauss-kernel-zonal-expansion}. Therefore, we can spectrally approximate it on datasets with bounded $\ell_2$ radius efficiently.

In particular, we first approximate $g(x,y)$ by a low-degree GZK and then invoke Theorem~\ref{thm:main-spectral-approx} on the resulting low-degree kernel. More precisely, for positive integers $s,q$ we let $g_{q,s}(x,y)$ be the order-$s$ GZK as per \cref{def-generalized-zonal-kernels} whose corresponding radial functions $h_\ell: \RR \to \RR^s$ are defined as follows for $i \in [s]$ and $\ell \le q$
\begin{equation}\label{def-monomial-coeff-gauss-kernel}
    [h_{\ell}(t)]_i = \sqrt{ \frac{\alpha_{\ell,d}}{2^\ell}  \frac{\Gamma(\frac{d}{2})}{\sqrt{\pi}(2i)!}  \frac{\Gamma(i+\frac{1}{2}) }{\Gamma(i+\ell+ \frac{ d}{2})} } \cdot t^{\ell+2i} e^{-\frac{t^2}{2}}
\end{equation}
and $h_\ell(t) \coloneqq 0$ for any $\ell > q$.
We show that $g_{q,s}(x,y)$ tightly approximates $g(x,y)$ for reasonably small values of $q$ and $s$, thus we can approximate the Gaussian kernel matrix by invoking \cref{thm:main-spectral-approx} on $g_{q,s}(x,y)$. Specifically, we prove,


\begin{restatable}{theorem}{thmspectralapproxgaussian} \label{thm-spectral-approx-Gaussian}
    Given $\X = [x_1, \dots, x_n] \in \RR^{d \times n}$ for $d\geq3$, assume that $\max_{j\in[n]}\norm{x_j} \le r$.
    Let $\K\in\RR^{n \times n}$ be the corresponding Gaussian kernel matrix $[\K]_{i,j} = e^{-\|x_i-x_j\|_2^2/2}$.  For any $0 < \lambda \le \norm{\K}_{\op}$ and $\varepsilon, \delta >0 $, let $s_\lambda$ denote the statistical dimension of $\K$ and define $q= \max \left\{ 3.7r^2, \frac{d}{2} \log\frac{2.8 (r^2 + \log\frac{n}{\varepsilon\lambda} + d)}{d} + \log \frac{n}{ \varepsilon \lambda} \right\}$. There exists an algorithm that can output a feature matrix $\Z \in \RR^{m \times n}$ with $m=\frac{5 q^2}{4\varepsilon^{2}} \binom{q + d - 1}{q} \log \left(\frac{16s_\lambda}{\delta}\right)$, such that with probability at least $1-\delta$, $\Z^\top \Z$ is an $(\varepsilon, \lambda)$-spectral approximation to $\K$ as per \cref{eq-spec-approx-truncated-features}. Furthermore, $\Z$ can be computed in time $\bigo((m/q) \cdot \nnz{\X})$.
\end{restatable}

The proof of \cref{thm-spectral-approx-Gaussian} is provided in \cref{sec-thm-spectral-approx-Gaussian}.
We remark that for any constant $\varepsilon = \Theta(1)$, dimension $d = \smallO \left( \log \frac{n}{\lambda} \right)$ and radius $r = \bigo \left(\sqrt{ \log \frac{n}{\lambda}} \right)$ our number of random features for spectrally approximating the Gaussian kernel matrix is sub-polynomial in $n/\lambda$. More precisely,
\begin{align*}
	m&=\bigo\left( \frac{\left( \frac{3d}{2} + \log \frac{n}{\lambda} \right)^{d} + (3.7r^2 + d)^{d}}{(d-1)!} \right) = \bigo\left( \frac{\left(2 \log \frac{n}{\lambda} \right)^{d} + (1.93r)^{2d}}{(d-1)!} \right) = ( n/\lambda )^{\smallO(1)}.
\end{align*}

This result improves upon prior works in a number of interesting ways. First, note that the only prior random features that can spectrally approximate the Gaussian kernel and is independent of the maximum norm of the input dataset is the random Fourier features~\cite{rahimi2007random}. Indeed, \citet{avron2017random} showed that spectral approximation can be achieved using random Fourier features. However, they also proved that the number of Fourier features should be at least $\Omega(n/\lambda)$, which is significantly larger than our number of features for any $d = \smallO \left( \log \frac{n}{\lambda} \right)$.

All other prior results on spectral approximation of the Gaussian kernel with features dimension that scales sub-linearly in $n/\lambda$, bear a dependence on the radius of the dataset, like our method. The modified Fourier features~\cite{avron2017random} assumes that the $\ell_\infty$-norm of all data points are bounded by some $r>0$ and constructs random features that spectrally approximate the Gaussian kernel matrix using 
\begin{align*}
	\bigo\left( \frac{(248r)^d \cdot (\log(n/\lambda))^{d/2} + (200 \log(n/\lambda))^{2d} }{\Gamma(d/2+1)} \right)	
\end{align*}
features. This is strictly larger than our number of features, by a large margin, for any radius $r = \bigo \left(\sqrt{ \log \frac{n}{\lambda}} \right)$ and any dimension $d$.

Additionally, there has been a line of work based on approximating the Gaussian kernel by low degree polynomials through Taylor expansion and then sketching the resulting polynomial. \citet{ahle2020oblivious} proposed a sketching method that runs in time $\bigo \left( r^{12} \cdot  (s_\lambda \cdot n + \nnz{\X} ) \cdot \text{poly}(\log(n/\lambda)) \right)$.
Additionally, \citet{woodruff2020near} improved the result of \citet{ahle2020oblivious} for high dimensional sparse datasets by combining sketching with adaptive sampling techniques. Their result runs in time 
$$\bigo \left( r^{15} \cdot  s_\lambda^2 \cdot n + r^{5} \cdot \nnz{\X} \cdot \text{poly}(\log(n/\lambda)) \right).$$ Because of the large exponent of the radius $r$, both of these bounds can easily become worse than our result for datasets with large radius in small constant dimensions $d = \bigo \left(1 \right)$. 
\cref{table-gauss-kernel-comparison} summarizes our result and all prior methods for approximating Gaussian kernel.

\begin{table}[t]
	\caption{Comparison of Gaussian kernel approximation algorithms in terms of feature dimension and runtime for $(\varepsilon, \lambda)$-spectral guarantee. The norm of dataset is bounded by $r$. We omit $(\log n)^{\bigo(1)}$ dependency for clarity and consider constant $\varepsilon$. We assume that $\max_i \norm{x_i} \leq r$.} \label{table-gauss-kernel-comparison}
	\vspace{0.05in}
	\centering
	\setlength{\tabcolsep}{24pt}
	\scalebox{0.9}{
		\begin{tabular}{@{}lll@{}}
			\hline
			& {Feature Dimension~($m$)} & Runtime \\
			\hline\hline
			\makecell[l]{Fourier \\ {\citep{rahimi2007random}}} & $\frac{n}{\lambda}$ & $m \cdot \nnz{\X} $  \\
			\hline
			\makecell[l]{Modified Fourier \\ {\citep{avron2017random}}} & 
			$(248r)^d (\log \frac{n}{\lambda})^{\frac{d}{2}} + (200 \log\frac{n}{\lambda})^{2d}$ &$m \cdot \nnz{\X}$ \\
			\hline
			\makecell[l]{Nystr\"om \\ {\citep{musco2017recursive}}} & ${s_\lambda}$ &$n m^2 + m \cdot \nnz{\X}$ \\
			\hline
			\makecell[l]{PolySketch \\ {\citep{ahle2020oblivious}}} & ${r^{10} \cdot s_\lambda}$ & ${r^{12}} \left( {n s_\lambda} + \nnz{\X} \right) $ \\
			\hline
			\makecell[l]{Adaptive Sketch \\ {\citep{woodruff2020near}}} & ${s_\lambda}$ &$r^{15} {s_\lambda^2} n+ r^{5} \nnz{\X}$ \\
			\hline
			\makecell[l]{{\bf Gegenbauer}\\ {(This work)}} &  $\displaystyle\frac{\left( 2 \log \frac{n}{\lambda}\right)^d + (1.93r)^{2d}}{(d-1)!}$ & $m \cdot \nnz{\X}$ \\
			\hline
		\end{tabular}
	}
\end{table}

\section{Experiments}

\subsection{Function Approximation via Gegenbauer Series} \label{sec-exp-gegen-poly-approx}
We first study function approximation of the Gegenbauer series for Gaussian and Neural Tangent Kernel of two-layer ReLU networks.
They correspond to function $\kappa(x) = \exp(2x)$ and $a_1(a_1(x)) + (a_1(x) + x a_0(x))\cdot a_0(a_1(x))$ for $x\in[-1,1]$ where $a_0(x)\coloneqq1-\frac{\mathrm{acos}(x)}{\pi}$ and $a_1(x)\coloneqq\frac{\sqrt{1-x^2} + x(\pi - \mathrm{acos}(x))}{\pi}$. We approximate these functions by Taylor, Chebyshev and Gegenbauer series with degree up to $15$ and compute approximation errors by $\max_{x \in [-1,1]} |\kappa(x) - \widetilde{\kappa}(x)|$ where $\widetilde{\kappa}$ is the polynomial approximation.  For the Gegenbauer, the dimension $d$ varies in $\{2,4,8,32\}$. Note that Taylor and Chebyshev are equivalent to Gegenbauer with $d=\infty$ and $2$, respectively. 
\cref{fig:ker-approx-error} shows that Gegenbauer series with a proper choice of $d$ provide better function approximators than the Taylor expansion. This can lead to performance improvement of the proposed random features, beyond Taylor series based kernel approximations, e.g., random Maclaurin~\cite{kar2012random} and polynomial sketch~\cite{ahle2020oblivious}.

\begin{figure}[t]
    \centering
    \subfigure[Gaussian]{\includegraphics[width=0.32\textwidth]{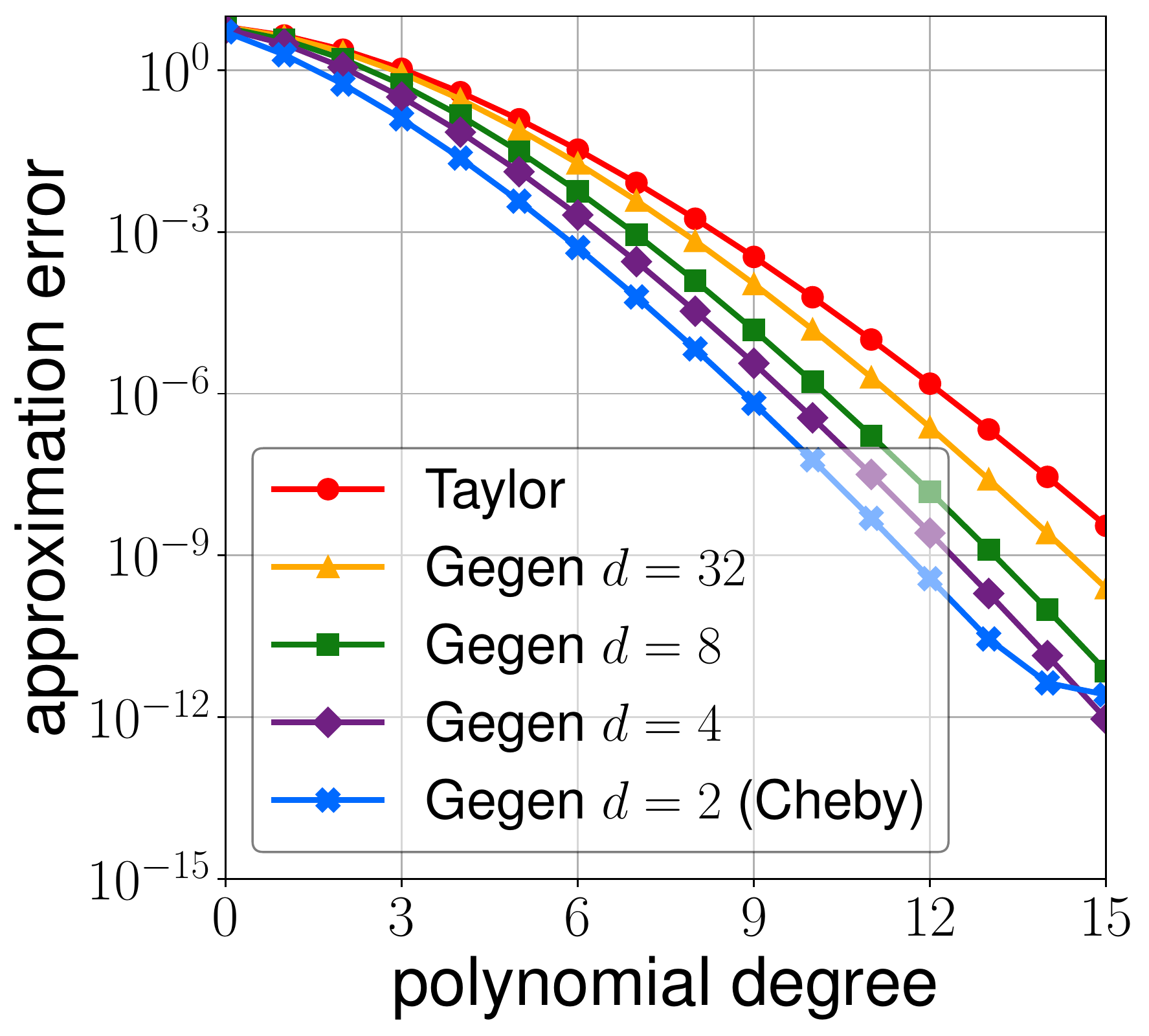}}
    \hspace{0.4in}
    \subfigure[Neural Tangent]{\includegraphics[width=0.32\textwidth]{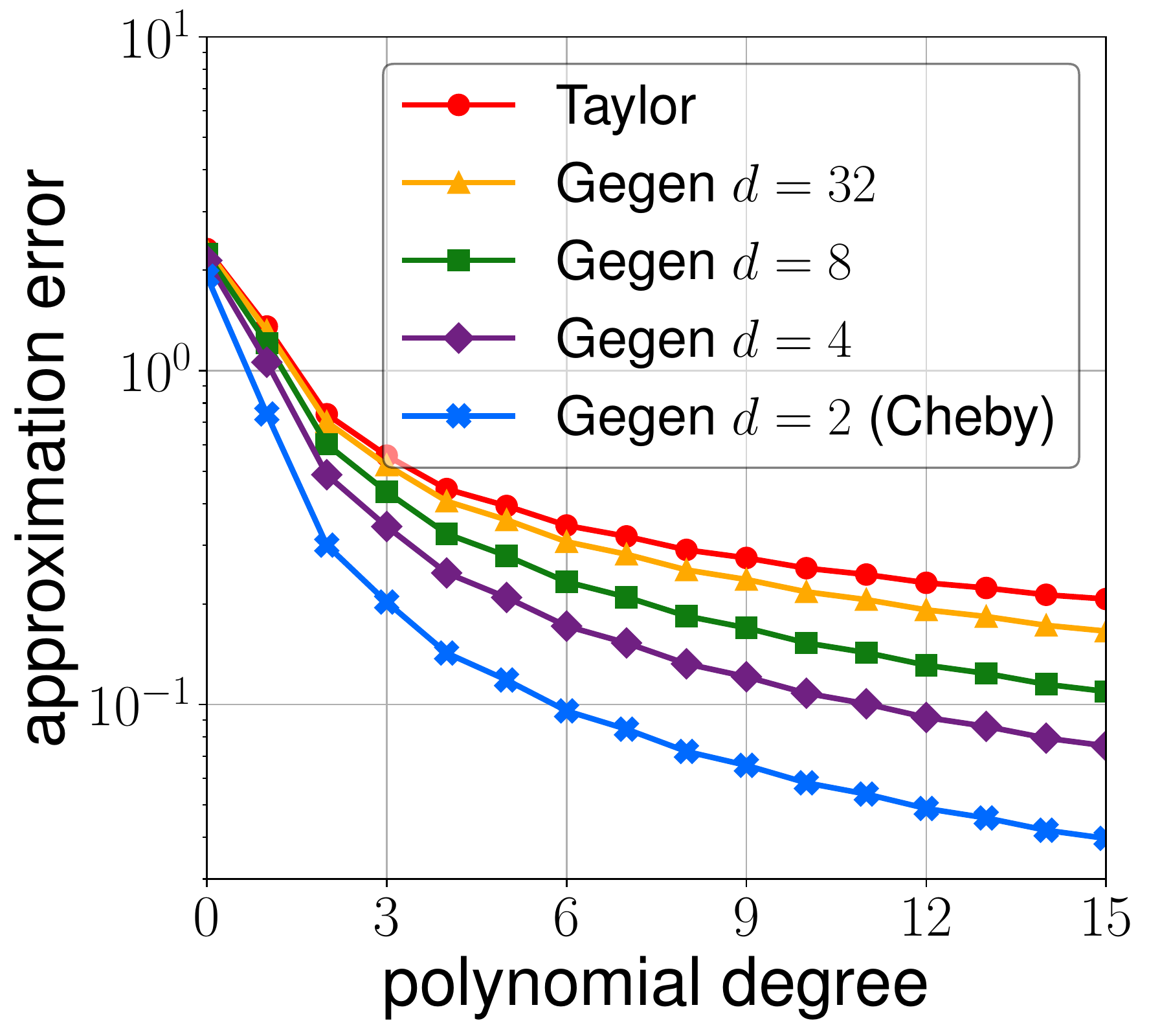}}
    \vspace{-0.15in}
    \caption{Kernel function approximation error of Taylor expansion and Gegenbauer expansion with $d \in \{2,4,8,32\}$. The case of $d=2$ is equivalent to the Chebyshev series expansion.} \label{fig:ker-approx-error}
\end{figure}

\newcommand{\coo}{\ensuremath{\mathrm{CO_2}~}}
\def\rff{Fourier}
\def\fst{FastFood}
\def\mac{Maclaurin}
\def\pts{PolySketch}
\def\nym{Nystr\"om}

\subsection{Kernel Ridge Regression}
Next we approximate kernel ridge regression on problems from $4$ real-world datasets, e.g., Earth Elevation, \coo, Climate and Protein. 
We consider the kernel ridge regression for predicting the outputs (e.g., earth elevation) with the Gaussian kernel. 
More details can be found in \cref{sec-exp-details-krr}.

We also benchmark various Gaussian kernel approximations including Nystr\"om~\cite{musco2017recursive}, Random Fourier Features~\cite{rahimi2007random} and that equipped with Hadamard transform~(known as FastFood)~\cite{le2013fastfood}, Random Maclaurin Features~\cite{kar2012random} and PolySketch~\cite{ahle2020oblivious}. We choose the feature dimension $m=1{,}024$ for all methods and datasets.
\cref{tab-regression-geospatial} summarizes the results. We observe that our proposed features (Gegenbauer) achieves the best both for \coo and climate datasets, and the second best for elevation. But, for Protein dataset whose dimension is larger than others, we verify that others show better performance. This follows from \cref{thm-spectral-approx-Gaussian} our methods requires large number of features when $d$ is large. Although the Nystr\"om method also performs well in practice, its runtime becomes much slower than ours.

\setlength{\tabcolsep}{10pt}
\begin{table}[t]
    \caption{Results of kernel ridge regression with Gaussian kernel.} \label{tab-regression-geospatial}
    \vspace{0.02in}
	\centering
	\scalebox{0.9}{
    \begin{tabular}{@{}lcccccccc@{}}
    	\hline
    	& \multicolumn{2}{c}{Elevation} & \multicolumn{2}{c}{\coo} & \multicolumn{2}{c}{Climate} & \multicolumn{2}{c}{Protein}\\
    	\hline
    	$n$ & \multicolumn{2}{c}{$64{,}800$} & \multicolumn{2}{c}{$146{,}040$} & \multicolumn{2}{c}{$223{,}656$} & \multicolumn{2}{c}{$45{,}730$} \\
    	Domain & \multicolumn{2}{c}{$\mathbb{S}^2$} & \multicolumn{2}{c}{$[\mathbb{S}^2, \mathbb{R}]$} & \multicolumn{2}{c}{$[\mathbb{S}^2, \mathbb{R}]$} & \multicolumn{2}{c}{$\mathbb{R}^9$} \\
    	Metric & MSE & Time & MSE & Time & MSE & Time & MSE & Time \\
    	\hline\hline
    	Nystrom & \textbf{1.14} & 3.81 & 0.533 & 8.17 & 3.14 & 12.0 & \textbf{18.9} & 2.85            \\
    	Fourier & 1.30 & 2.10 & 0.548 & 4.73 & 3.15 & 6.93 & 19.8 & 1.66 \\
    	FastFood & 1.35 & 7.79 & 0.551 & 17.3 & 3.16 & 26.3 & 19.8 & 4.94 \\
    	Maclaurin & 1.90 & 1.07 & 0.593   & 2.38  & 3.18   & 3.55   & 25.9   & 1.05   \\
    	PolySketch & 1.56  & 7.65 & 0.590   & 16.4  & 3.15   & 23.5   & 26.9   & 4.96   \\ \midrule
    	Gegenbauer  & 1.15  & 1.71 & {\bf 0.532}   & 3.49  & \textbf{3.13}   & 5.41   & 21.0   & 9.72   \\
        \hline
    \end{tabular}
	}
\end{table}

\begin{table}[t]
	\caption{$k$-means clustering objective with the Gaussian kernel.} \label{table-kmeans-cluster}
	\vspace{0.02in}
	\centering
	\scalebox{0.9}{
		\begin{tabular}{@{}lcccccc@{}}
			\hline
			& {Abalone} & {Pendigits} & {Mushroom} & {Magic} & Statlog & Connect-4 \\
			$n$  & 4{,}177 & 7{,}494 & 8{,}124 & 19{,}020 & 43{,}500 & 67{,}557 \\
			$d$  & 8 & 16 & 21 & 10 & 9 & 42 \\
			\hline\hline
			Nystr\"om    & 0.38 & 0.42 & 0.71 & 0.64 & 0.23 & {\bf 0.61}  \\
			Fourier      & 0.38 & 0.43 & 0.72 & 0.66  & 0.24 & 0.81  \\
			FastFood     & 0.43 & 0.46 & 0.74 & 0.67  & 0.24 & 0.83  \\
			Maclaurin    & 0.43 & 0.46 & 0.72 & 0.73 & 0.23 & 0.90  \\
			PolySketch   & {\bf 0.35} & 0.45 & {\bf 0.67} & 0.66 & {\bf 0.21} & 0.82 \\
			\hline
			Gegenbauer     & {\bf 0.35} & {\bf 0.40} & 0.71 & {\bf 0.59} & {\bf 0.21} & 0.78  \\
			\hline
		\end{tabular}
	}
\end{table}


\subsection{Kernel $k$-means Clustering}
We apply the proposed random features to kernel $k$-means clustering under $6$ UCI classification datasets. We choose the Gaussian kernel and explore various approximating algorithms as described above where feature dimension is set to $m=512$. 
We evaluate the average summation of squared distance to the nearest cluster centers. Formally, given data points $x_1, \dots, x_n$, let $\phi_i$ be some feature map of $x_i$ and denote $\mu_i = \frac{1}{|C_i|} \sum_{x_j \in C_i} \phi_{x_i}$ be the centroid of the vectors in $C_i$ after mapping to kernel space.  The goal of kernel $k$-means is to choose partitions $\{ C_1, \ldots ,C_k \}$ which minimize the following objective:
$\sum_{i=1}^k \sum_{x_j \in C_i} \| \phi_{x_j} - \mu_i \|_2^2$.
\cref{table-kmeans-cluster} reports the result of $k$-means clustering. We observe that our random Gegenbauer features shows promising performances except Mushroom and Connect-4 datasets, which have a higher input dimension. More details are in \cref{sec-exp-details-kmeans}.

\section{Conclusion}
We proposed a new class of kernel functions expressed by Gegenbauer polynomials which cover a wide range of ubiquitous kernel functions, such as Gaussian and all dot-product kernels. Moreover, we proposed random features for speeding up kernel-based learning methods, which can spectrally approximate kernel matrices. Our random features can tightly approximate the kernel matrices when the input points are in a low-dimensional space, however in high dimensions our method performs less efficiently. We believe that this can be alleviated when our method is combined with additional dimensionality reductions (e.g., JL-transform). We leave open the question for high-dimensional inputs for future work.

\section*{Acknowledgements}

Haim Avron was partially supported by the Israel Science Foundation (grant no. 1272/17) and by the US-Israel Binational Science Foundation (grant no. 2017698). 
Amir Zandieh was supported by the Swiss NSF grant No. P2ELP2\textunderscore 195140. Insu Han was supported by TATA DATA Analysis (grant no. 105676).

\bibliography{references}

\newcommand{\etalchar}[1]{$^{#1}$}
\begin{thebibliography}{AKM{\etalchar{+}}19}

\bibitem[ACW17]{avron2017faster}
Haim Avron, Kenneth~L. Clarkson, and David~P. Woodruff.
\newblock {\href{https://epubs.siam.org/doi/abs/10.1137/16M1105396}{Faster
  Kernel Ridge Regression Using Sketching and Preconditioning}}.
\newblock In {\em SIAM Journal on Matrix Analysis and Applications (SIMAX)},
  2017.

\bibitem[AH12]{atkinson2012spherical}
Kendall Atkinson and Weimin Han.
\newblock {\em {\href{https://www.springer.com/gp/book/9783642259821}{Spherical
  harmonics and approximations on the unit sphere: an introduction}}}.
\newblock Springer Science \& Business Media, 2012.

\bibitem[AKK{\etalchar{+}}20]{ahle2020oblivious}
Thomas~D Ahle, Michael Kapralov, Jakob~BT Knudsen, Rasmus Pagh, Ameya
  Velingker, David~P Woodruff, and Amir Zandieh.
\newblock {\href{https://arxiv.org/pdf/1909.01410.pdf}{Oblivious sketching of
  high-degree polynomial kernels}}.
\newblock In {\em Symposium on Discrete Algorithms (SODA)}, 2020.

\bibitem[AKM{\etalchar{+}}17]{avron2017random}
Haim Avron, Michael Kapralov, Cameron Musco, Christopher Musco, Ameya
  Velingker, and Amir Zandieh.
\newblock {\href{https://arxiv.org/pdf/1804.09893.pdf}{Random Fourier features
  for kernel ridge regression: Approximation bounds and statistical
  guarantees}}.
\newblock In {\em International Conference on Machine Learning (ICML)}, 2017.

\bibitem[AKM{\etalchar{+}}19]{avron2019universal}
Haim Avron, Michael Kapralov, Cameron Musco, Christopher Musco, Ameya
  Velingker, and Amir Zandieh.
\newblock {\href{https://arxiv.org/pdf/1812.08723.pdf}{A universal sampling
  method for reconstructing signals with simple fourier transforms}}.
\newblock In {\em Symposium on the Theory of Computing (STOC)}, 2019.

\bibitem[AM15]{alaoui2014fast}
Ahmed~El Alaoui and Michael~W Mahoney.
\newblock {\href{https://arxiv.org/pdf/1411.0306.pdf}{Fast randomized kernel
  methods with statistical guarantees}}.
\newblock In {\em Neural Information Processing Systems (NeurIPS)}, 2015.

\bibitem[ANW14]{avron2014subspace}
Haim Avron, Huy Nguyen, and David Woodruff.
\newblock
  {\href{https://papers.nips.cc/paper/2014/file/b571ecea16a9824023ee1af16897a582-Paper.pdf}{Subspace
  embeddings for the polynomial kernel}}.
\newblock In {\em Neural Information Processing Systems (NeurIPS)}, 2014.

\bibitem[AV06]{arthur2006k}
David Arthur and Sergei Vassilvitskii.
\newblock {\href{http://ilpubs.stanford.edu:8090/778/1/2006-13.pdf}{k-means++:
  The advantages of careful seeding}}.
\newblock Technical report, Stanford, 2006.

\bibitem[Bac13]{bach2013sharp}
Francis Bach.
\newblock {\href{https://arxiv.org/pdf/1208.2015.pdf}{Sharp analysis of
  low-rank kernel matrix approximations}}.
\newblock In {\em Conference on Learning Theory (COLT)}, 2013.

\bibitem[Ber03]{bercovici2003generation}
David Bercovici.
\newblock
  {\href{https://www.sciencedirect.com/science/article/pii/S0012821X02010099}{The
  generation of plate tectonics from mantle convection}}.
\newblock {\em Earth and Planetary Science Letters}, 2003.

\bibitem[CMM17]{cohen2017input}
Michael~B Cohen, Cameron Musco, and Christopher Musco.
\newblock {\href{https://arxiv.org/pdf/1511.07263.pdf}{Input sparsity time
  low-rank approximation via ridge leverage score sampling}}.
\newblock In {\em Symposium on Discrete Algorithms (SODA)}, 2017.

\bibitem[CS09]{cho2009kernel}
Youngmin Cho and Lawrence Saul.
\newblock {\href{https://cseweb.ucsd.edu/~saul/papers/nips09_kernel.pdf}{Kernel
  methods for deep learning}}.
\newblock In {\em Neural Information Processing Systems (NeurIPS)}, 2009.

\bibitem[CS10]{cho2010large}
Youngmin Cho and Lawrence~K Saul.
\newblock
  {\href{https://direct.mit.edu/neco/article-abstract/22/10/2678/7579/Large-Margin-Classification-in-Infinite-Neural?redirectedFrom=fulltext}{Large-margin
  classification in infinite neural networks}}.
\newblock {\em Neural computation}, 2010.

\bibitem[DGK04]{dhillon2004kernel}
Inderjit~S Dhillon, Yuqiang Guan, and Brian Kulis.
\newblock {\href{https://dl.acm.org/doi/pdf/10.1145/1014052.1014118}{Kernel
  k-means: spectral clustering and normalized cuts}}.
\newblock In {\em Conference on Knowledge Discovery and Data Mining (KDD)},
  2004.

\bibitem[FP68]{fox1968chebyshev}
Leslie Fox and Ian~Bax Parker.
\newblock
  {\href{https://www.cambridge.org/core/journals/mathematical-gazette/article/abs/chebyshev-polynomials-in-numerical-analysis-by-l-fox-and-i-b-parker-pp-ix-205-42s-1968-oxford/76D763992002A603EB82AFF1E76BBFA3}{Chebyshev
  polynomials in numerical analysis}}.
\newblock Technical report, 1968.

\bibitem[Gau04]{gautschi2004orthogonal}
Walter Gautschi.
\newblock {\em
  {\href{http://213.230.96.51:8090/files/ebooks/Matematika/Gautschi\%20W.\%20Orthogonal\%20polynomials..\%20computation\%20and\%20approximation\%20(ISBN\%200198506724)(OUP,\%202004)(O)(314s)\%20MCsf\%20.pdf}{Orthogonal
  polynomials: computation and approximation}}}.
\newblock OUP Oxford, 2004.

\bibitem[JGH18]{jacot2018neural}
Arthur Jacot, Franck Gabriel, and Cl{\'e}ment Hongler.
\newblock {\href{https://arxiv.org/pdf/1806.07572.pdf}{Neural tangent kernel:
  Convergence and generalization in neural networks}}.
\newblock In {\em Neural Information Processing Systems (NeurIPS)}, 2018.

\bibitem[KK12]{kar2012random}
Purushottam Kar and Harish Karnick.
\newblock {\href{https://arxiv.org/pdf/1201.6530.pdf}{Random feature maps for
  dot product kernels}}.
\newblock In {\em Conference on Artificial Intelligence and Statistics
  (AISTATS)}, 2012.

\bibitem[LJS16]{li2016fast}
Chengtao Li, Stefanie Jegelka, and Suvrit Sra.
\newblock {\href{http://proceedings.mlr.press/v48/lih16.pdf}{Fast dpp sampling
  for nystrom with application to kernel methods}}.
\newblock In {\em International Conference on Machine Learning (ICML)}, 2016.

\bibitem[LMP13]{li2013iterative}
Mu~Li, Gary~L Miller, and Richard Peng.
\newblock {\href{https://arxiv.org/pdf/1211.2713.pdf}{Iterative row sampling}}.
\newblock In {\em Foundations of Computer Science (FOCS)}, 2013.

\bibitem[LSS{\etalchar{+}}13]{le2013fastfood}
Quoc Le, Tam{\'a}s Sarl{\'o}s, Alex Smola, et~al.
\newblock {\href{https://arxiv.org/pdf/1408.3060.pdf}{Fastfood-approximating
  kernel expansions in loglinear time}}.
\newblock In {\em International Conference on Machine Learning (ICML)}, 2013.

\bibitem[LXS{\etalchar{+}}19]{lee2019wide}
Jaehoon Lee, Lechao Xiao, Samuel Schoenholz, Yasaman Bahri, Roman Novak, Jascha
  Sohl-Dickstein, and Jeffrey Pennington.
\newblock {\href{https://arxiv.org/abs/1902.06720}{Wide neural networks of any
  depth evolve as linear models under gradient descent}}.
\newblock In {\em Neural Information Processing Systems (NeurIPS)}, 2019.

\bibitem[MH02]{mason2002chebyshev}
John~C Mason and David~C Handscomb.
\newblock {\em
  {\href{https://www.taylorfrancis.com/books/mono/10.1201/9781420036114/chebyshev-polynomials-mason-david-handscomb}{Chebyshev
  polynomials}}}.
\newblock CRC press, 2002.

\bibitem[MM17]{musco2017recursive}
Cameron Musco and Christopher Musco.
\newblock {\href{https://arxiv.org/pdf/1605.07583.pdf}{Recursive Sampling for
  the Nystr\"om Method}}.
\newblock In {\em Neural Information Processing Systems (NeurIPS)}, 2017.

\bibitem[MNY06]{minh2006mercer}
Ha~Quang Minh, Partha Niyogi, and Yuan Yao.
\newblock
  {\href{http://people.cs.uchicago.edu/~niyogi/papersps/MinNiyYao06.pdf}{Mercer’s
  theorem, feature maps, and smoothing}}.
\newblock In {\em Conference on Learning Theory (COLT)}, 2006.

\bibitem[Mor98]{morimoto1998analytic}
Mitsuo Morimoto.
\newblock {\em Analytic functionals on the sphere}.
\newblock American Mathematical Soc., 1998.

\bibitem[Oga88]{ogawa1988operator}
Hidemitsu Ogawa.
\newblock {\href{https://epubs.siam.org/doi/pdf/10.1137/0148095}{An operator
  pseudo-inversion lemma}}.
\newblock {\em SIAM Journal on Applied Mathematics}, 1988.

\bibitem[PD16]{paul2016feature}
Saurabh Paul and Petros Drineas.
\newblock
  {\href{https://www.cs.purdue.edu/homes/pdrineas/documents/publications/Drineas_NECO_2016.pdf}{Feature
  selection for ridge regression with provable guarantees}}.
\newblock {\em Neural computation}, 2016.

\bibitem[PP13]{pham2013fast}
Ninh Pham and Rasmus Pagh.
\newblock {\href{http://chbrown.github.io/kdd-2013-usb/kdd/p239.pdf}{Fast and
  scalable polynomial kernels via explicit feature maps}}.
\newblock In {\em Conference on Knowledge Discovery and Data Mining (KDD)},
  2013.

\bibitem[PYK15]{pennington2015spherical}
Jeffrey Pennington, Felix Xinnan~X Yu, and Sanjiv Kumar.
\newblock
  {\href{https://papers.nips.cc/paper/5943-spherical-random-features-for-polynomial-kernels.pdf}{Spherical
  random features for polynomial kernels}}.
\newblock In {\em Neural Information Processing Systems (NeurIPS)}, 2015.

\bibitem[Ras04]{rasmussen2004gaussian}
Carl~Edward Rasmussen.
\newblock {\href{http://www.gaussianprocess.org/gpml/chapters/}{Gaussian
  processes in machine learning}}.
\newblock In {\em Advanced lectures on machine learning}. Springer, 2004.

\bibitem[RR09]{rahimi2007random}
Ali Rahimi and Benjamin Recht.
\newblock
  {\href{https://people.eecs.berkeley.edu/~brecht/papers/07.rah.rec.nips.pdf}{Random
  Features for Large-Scale Kernel Machines}}.
\newblock In {\em Neural Information Processing Systems (NeurIPS)}, 2009.

\bibitem[SAZ{\etalchar{+}}20]{su2020prediction}
Xiaoqian Su, Junlin An, Yuxin Zhang, Ping Zhu, and Bin Zhu.
\newblock
  {\href{https://www.semanticscholar.org/paper/Prediction-of-ozone-hourly-concentrations-by-vector-Su-An/db22e1adaf13448ed098e45531ae553a1bb19d07}{Prediction
  of ozone hourly concentrations by support vector machine and kernel extreme
  learning machine using wavelet transformation and partial least squares
  methods}}.
\newblock {\em Atmospheric Pollution Research}, 2020.

\bibitem[Sch42]{schoenberg1988positive}
I~Schoenberg.
\newblock
  {\href{https://projecteuclid.org/journals/duke-mathematical-journal/volume-9/issue-1/Positive-definite-functions-on-spheres/10.1215/S0012-7094-42-00908-6.full}{Positive
  definite functions on spheres}}.
\newblock {\em Duke Math. J}, 1942.

\bibitem[SGV98]{saunders1998ridge}
Craig Saunders, Alexander Gammerman, and Volodya Vovk.
\newblock {\href{https://eprints.soton.ac.uk/258942/1/Dualrr_ICML98.pdf}{Ridge
  regression learning algorithm in dual variables}}.
\newblock In {\em International Conference on Machine Learning (ICML)}, 1998.

\bibitem[SH21]{scetbon2021spectral}
Meyer Scetbon and Zaid Harchaoui.
\newblock {\href{http://proceedings.mlr.press/v130/scetbon21b/scetbon21b.pdf}{A
  Spectral Analysis of Dot-product Kernels}}.
\newblock In {\em Conference on Artificial Intelligence and Statistics
  (AISTATS)}, 2021.

\bibitem[SOW{\etalchar{+}}01]{smola2001regularization}
Alex~J Smola, Zoltan~L Ovari, Robert~C Williamson, et~al.
\newblock
  {\href{https://alex.smola.org/papers/2001/SmoOvaWil01.pdf}{Regularization
  with dot-product kernels}}.
\newblock {\em Neural Information Processing Systems (NeurIPS)}, 2001.

\bibitem[SSI10]{sanderson2010climate}
Benjamin~M Sanderson, Karen~M Shell, and William Ingram.
\newblock
  {\href{https://link.springer.com/content/pdf/10.1007/s00382-009-0661-1.pdf}{Climate
  feedbacks determined using radiative kernels in a multi-thousand member
  ensemble of AOGCMs}}.
\newblock {\em Climate dynamics}, 2010.

\bibitem[SWYZ21]{song2021fast}
Zhao Song, David Woodruff, Zheng Yu, and Lichen Zhang.
\newblock {\href{http://proceedings.mlr.press/v139/song21c/song21c.pdf}{Fast
  sketching of polynomial kernels of polynomial degree}}.
\newblock In {\em International Conference on Machine Learning (ICML)}, 2021.

\bibitem[Tro15]{tropp2015introduction}
Joel~A Tropp.
\newblock {\href{https://arxiv.org/pdf/1501.01571.pdf}{An introduction to
  matrix concentration inequalities}}.
\newblock {\em arXiv preprint arXiv:1501.01571}, 2015.

\bibitem[VSKB10]{vishwanathan2010graph}
S~Vichy~N Vishwanathan, Nicol~N Schraudolph, Risi Kondor, and Karsten~M
  Borgwardt.
\newblock
  {\href{https://www.jmlr.org/papers/volume11/vishwanathan10a/vishwanathan10a.pdf}{Graph
  kernels}}.
\newblock {\em Journal of Machine Learning Research (JMLR)}, 2010.

\bibitem[WS01]{williams2001using}
Christopher Williams and Matthias Seeger.
\newblock
  {\href{https://papers.nips.cc/paper/2000/file/19de10adbaa1b2ee13f77f679fa1483a-Paper.pdf}{Using
  the Nystr{\"o}m method to speed up kernel machines}}.
\newblock In {\em Neural Information Processing Systems (NeurIPS)}, 2001.

\bibitem[WZ20]{woodruff2020near}
David~P Woodruff and Amir Zandieh.
\newblock {\href{https://arxiv.org/pdf/2007.03927.pdf}{Near Input Sparsity Time
  Kernel Embeddings via Adaptive Sampling}}.
\newblock In {\em International Conference on Machine Learning (ICML)}, 2020.

\bibitem[YLM{\etalchar{+}}12]{yang2012nystrom}
Tianbao Yang, Yu-Feng Li, Mehrdad Mahdavi, Rong Jin, and Zhi-Hua Zhou.
\newblock
  {\href{https://papers.nips.cc/paper/2012/file/621bf66ddb7c962aa0d22ac97d69b793-Paper.pdf}{Nystr{\"o}m
  method vs random fourier features: A theoretical and empirical comparison}}.
\newblock In {\em Neural Information Processing Systems (NeurIPS)}, 2012.

\bibitem[ZHA{\etalchar{+}}21]{zandieh2021scaling}
Amir Zandieh, Insu Han, Haim Avron, Neta Shoham, Chaewon Kim, and Jinwoo Shin.
\newblock {\href{https://openreview.net/forum?id=vIRFiA658rh}{Scaling Neural
  Tangent Kernels via Sketching and Random Features}}.
\newblock In {\em Neural Information Processing Systems (NeurIPS)}, 2021.

\bibitem[ZNV{\etalchar{+}}20]{zandieh2020scaling}
Amir Zandieh, Navid Nouri, Ameya Velingker, Michael Kapralov, and Ilya
  Razenshteyn.
\newblock {\href{https://arxiv.org/pdf/2003.09756.pdf}{Scaling up kernel ridge
  regression via locality sensitive hashing}}.
\newblock In {\em Conference on Artificial Intelligence and Statistics
  (AISTATS)}, 2020.

\end{thebibliography}
\bibliographystyle{alpha}

\newpage
\appendix

\section{Applications to Learning Tasks} \label{sec-learning-tasks}

In this section, we prove that our general kernel approximation guarantees from \cref{thm:main-spectral-approx} and \cref{thm:main-proj-cost-preserv-approx} are sufficient for many downstream learning tasks without sacrificing accuracy or statistical performance of our random features.

\subsection{Kernel Ridge Regression} \label{sec-learning-tasks-krr}
One way to analyze the quality of approximate kernel ridge regression (KRR) estimator is by bounding the excess risk compared to the exact KRR estimator. We consider a fixed design setting which has been particularly popular in analysis of KRR~\cite{bach2013sharp, alaoui2014fast, li2016fast, paul2016feature, musco2017recursive, avron2017random,zandieh2020scaling}. In this setting, we assume that our observed labels $y_i$ represent some underlying true labels $f^*(x_i)$ perturbed with Gaussian noise with variance $\sigma^2$. More specifically, we assume $y_i$ satisfies
\[ y_i = f^*(x_i) + \nu_i \]
for some $f^*:\RR^d \to \RR$. Then, the empirical risk of an estimator $f$ is defined as
\begin{align}
    \mathcal{R}(f) \coloneqq \E_{\{v_i\}_{i=1}^n} \left[ \frac{1}{n} \sum_{i=1}^n \abs{f(x_i) - f^*(x_i)}^2 \right]
\end{align}
Given this definition of risk, our \cref{thm:main-spectral-approx} along with \citep[Lemma~2]{avron2017random} immediately gives the following bound on the risk of approximate KRR using our feature matrix $\BPhi$,
\begin{lemma}[Kernel ridge regression risk bound]
    Given that preconditions of \cref{thm:main-spectral-approx} hold, let $f$ be the exact KRR estimator using kernel $\K + \lambda \I$ and $\tilde{f}$ be the approximate estimator obtained using the approximate kernel $\Z^\top \Z + \lambda \I$.
    If $\norm{\K}_\op \geq 1$ and $\Z^\top \Z$ is an $(\varepsilon,\lambda)$-spectral approximation to $\K$ for some $0\le \varepsilon < 1$ as per \eqref{eq-spec-approx-truncated-features} then 
    \[ \mathcal{R}(\tilde{f}) \le \frac{\mathcal{R}(f)}{1-\varepsilon} + \frac{\varepsilon}{1+\varepsilon}\cdot \frac{\rank(\Z)}{n} \cdot \sigma^2. \]
\end{lemma}

\subsection{Kernel $k$-means clustering.}
Kernel $k$-means clustering aims at partitioning the data-points $x_1, \cdots , x_n \in \RR^d$, into $k$ cluster sets, $\{C_1, \ldots ,C_k\}$ such that the sum of squares of \emph{kernel distances} of data-points from their associated cluster center is minimized. Specifically, for our generalized zonal kernel function (\cref{def-generalized-zonal-kernels}), if we let $\mu_i = \frac{1}{|C_i|} \sum_{x_j \in C_i} \phi_{x_i}$ be the centroid of the vectors in $C_i$ after mapping to kernel space using the feature map $\phi_x$ defined in \cref{lem-gzk-feature-map}, then the goal of kernel $k$-means is to choose partitions $\{ C_1, \dots ,C_k \}$ which minimize the following objective:
\[ \sum_{i=1}^k \sum_{x_j \in C_i} \| \phi_{x_j} - \mu_i \|_{L^2(S^{d-1}, \RR^s)}^2. \]
This optimization problem can be rewritten as a constrained low-rank approximation problem~\cite{musco2017recursive}. In particular, for any clustering $\{ C_1, \dots, C_k \}$ we can define a rank-$k$ orthonormal matrix $\C \in \RR^{n \times k}$, called the cluster indicator matrix, as $\C_{j, i} \coloneqq \frac{1}{|C_i|} \cdot \mathbbm{1}_{\{ x_j \in C_i \}}$ for every $i \in [k]$ and $j \in [n]$. Note that with this definition we have $\C^\top \C = \I_k$, so $\C \C^\top$ is a rank $k$ projection matrix. Therefore, if we let $\K \in \RR^{n \times n} $ be the GZK kernel matrix, the kernel $k$-means cost function is equivalent to
\[ \sum_{i=1}^k \sum_{x_j \in C_i} \| \phi_{x_j} - \mu_i \|_{L^2(S^{d-1}, \RR^s)}^2 \coloneqq \trace\left( \K - \C \C^\top \K \C \C^\top \right). \]
Thus we can approximately solve this problem by using our random features $\Z$ constructed in \cref{def-random-feature-construction} and solving the following problem:
\[ \min_{\text{cluster indicator }\C} \| \Z - \Z \C \C^\top \|_F^2. \]
Specifically, using our \cref{thm:main-proj-cost-preserv-approx} along with \citep[Theorem~16]{musco2017recursive} we have the following approximation bound,

\begin{lemma}
    Given that preconditions of \cref{thm:main-proj-cost-preserv-approx} hold, if we let $\widetilde{\C} \in \RR^{n \times k}$ be an approximately optimal cluster indicator matrix for the following $k$-means problem,
    \[ \| \Z - \Z \widetilde{\C} \widetilde{\C}^\top \|_F^2 \le (1 + \gamma)\min_{\text{cluster indicator }\C} \| \Z - \Z \C \C^\top \|_F^2, \]
    for some $\gamma \ge 0$, then we have the following,
    \[ \| \Z - \Z \widetilde{\C} \widetilde{\C}^\top \|_F^2 \le (1 + \gamma) (1 + \varepsilon) \min_{\text{cluster indicator }\C} \trace\left( \K - \C \C^\top \K \C \C^\top \right). \]
\end{lemma}

\section{Class of GZKs Contains All Dot-product Kernels}\label{sec-proof-lem-dot-prod-spherical-harmonic-expansion}
In this section we prove \cref{lem-dot-prod-spherical-harmonic-expansion}, which implies that the class of GZK given in \cref{def-generalized-zonal-kernels} includes all dot-product kernels.

\lmmdotprodaregzk*

\begin{proofof}{\cref{lem-dot-prod-spherical-harmonic-expansion}}
We begin with the Taylor series expansion of the function $\kappa(\cdot)$ around zero. Because $\kappa(\cdot)$ is analytic, the series expansion exists and converges to $\kappa$. So we have,
\begin{align}
	\kappa(\langle x, y \rangle) &\coloneqq \sum_{j=0}^\infty \frac{\kappa^{(j)}(0)}{j!} \cdot \langle x, y \rangle^j
	= \sum_{j=0}^\infty \frac{\kappa^{(j)}(0)}{j!} \cdot \|x\|^j \cdot \|y\|^j \cdot  \left(\frac{\langle x, y \rangle}{\|x\| \cdot \|y\|}\right)^j.\label{eq:taylo-series-kappa}
\end{align}
Now we write the degree-$j$ monomial $t^j$ for any integer $j \ge 0$, in the basis of $d$-dimensional Gegenbauer polynomials, $P_{d}^0(t), P_{d}^1(t), P_{d}^2(t), \ldots P_{d}^j(t)$. More precisely, by \cref{eq-gegen-expansion-coeffs}, we find $t^j \coloneqq \sum_{\ell=0}^j \mu_\ell^j \cdot P_d^{\ell}(t)$ where 
\begin{align} 
	\mu_\ell^j = \alpha_{\ell,d} \cdot \frac{\left| \SS^{d-2} \right|}{\left| \SS^{d-1} \right|} \int_{-1}^{1} t^j \cdot P_d^{\ell}(t) \cdot (1-t^2)^{\frac{d-3}{2}} dt.
\end{align}
By using the Rodrigues' formula in \cref{eq:Gegenbauer-poly-derivate-def}, we can compute the Gegenbauer coefficients of $t^j$ as follows,
\begin{equation}\label{eq:gegenbauer-expansion-monomial}
	\mu_\ell^j = \alpha_{\ell,d} \cdot \frac{(-1)^\ell}{2^\ell}\cdot \frac{\left| \SS^{d-2} \right|}{\left| \SS^{d-1} \right|} \cdot  \frac{\Gamma\left( \frac{d-1}{2}\right)}{\Gamma\left( \ell + \frac{d-1}{2}\right)} 
	\int_{-1}^{1} t^j \cdot  \frac{d^\ell}{dt^\ell} \left( 1-t^2 \right)^{\ell + \frac{d-3}{2}} \, dt.
\end{equation}
By multiple applications of integration by parts we can compute the integral in \cref{eq:gegenbauer-expansion-monomial} as follows,
\begin{align}
	\int_{-1}^{1} t^j \cdot  \frac{d^\ell \left( 1-t^2 \right)^{\ell + \frac{d-3}{2}} }{dt^\ell} \, dt &= \left. t^j \cdot \frac{d^{\ell-1} (1-t^2)^{\ell + \frac{d-3}{2}}}{dt^{\ell-1}}  \right|_{-1}^{1} - j \int_{-1}^1 t^{j-1} \cdot \frac{d^{\ell-1} (1-t^2)^{\ell + \frac{d-3}{2}}}{dt^{\ell-1}} \, dt\nonumber \\
	&= - j \int_{-1}^1 t^{j-1} \cdot \frac{d^{\ell-1} (1-t^2)^{\ell + \frac{d-3}{2}}}{dt^{\ell-1}} \, dt\nonumber \\
	&= \left. - j t^{j-1} \cdot \frac{d^{\ell-2} (1-t^2)^{\ell + \frac{d-3}{2}}}{dt^{\ell-2}}  \right|_{-1}^{1} + j(j-1)\int_{-1}^1  t^{j-2} \cdot \frac{d^{\ell-2} (1-t^2)^{\ell + \frac{d-3}{2}}}{dt^{\ell-2}} dt\nonumber \\
	&= (-1)^2 \cdot j (j-1) \int_{-1}^1  t^{j-2} \cdot \frac{d^{\ell-2} (1-t^2)^{\ell + \frac{d-3}{2}}}{dt^{\ell-2}}  dt \nonumber\\
	&~~\vdots \nonumber\\
	&= (-1)^\ell \cdot  \frac{j!}{(j-\ell)!} \int_{-1}^1  t^{j-\ell} \cdot (1-t^2)^{\ell + \frac{d-3}{2}} dt. \label{eq:genegbauer-monmial-integral-by-parts}
\end{align}

Now note that the above integral is zero if $j-\ell$ is an odd integer. So, we focus on the cases where $j-\ell$ is an even integer. By a change of variables to $u = t^2$ we have,
\begin{align*}
	\int_{-1}^1  t^{j-\ell} \cdot (1-t^2)^{\ell + \frac{d-3}{2}} dt &= \int_{0}^1  u^{\frac{j-\ell-1}{2}} \cdot (1-u)^{\ell + \frac{d-3}{2}} du
	= \frac{\Gamma(\frac{j-\ell+1}{2}) \cdot \Gamma(\ell + \frac{d-1}{2})}{\Gamma(\frac{j+\ell + d}{2})}.
\end{align*}
By combining the above with \cref{eq:genegbauer-monmial-integral-by-parts} and \cref{eq:gegenbauer-expansion-monomial} and using the fact that $\frac{\left| \SS^{d-2} \right|}{\left| \SS^{d-1} \right|} = \frac{\Gamma(\frac{d}{2})}{\sqrt{\pi} \cdot \Gamma(\frac{d-1}{2})}$, we find the following
\begin{equation}\label{eq:gegenbauer-expansion-taylor-series-kappa}
	\mu_\ell^j = \begin{dcases} 
		\frac{\alpha_{\ell,d}}{2^\ell}\cdot \frac{\Gamma(\frac{d}{2}) \cdot j!}{\sqrt{\pi} \cdot (j-\ell)!} \cdot \frac{\Gamma(\frac{j-\ell+1}{2}) }{\Gamma(\frac{j+ \ell+ d}{2})} & \text{ if $j-\ell$ is even}\\
		0 & \text{ if $j-\ell$ is odd}
	\end{dcases}
\end{equation}

Now if we plug the monomial expansion $t^j \coloneqq \sum_{\ell=0}^j \mu_\ell^j \cdot P_d^{\ell}(t)$ into \cref{eq:taylo-series-kappa}, using the fact that $\mu_\ell^j = 0$ for any odd $j-\ell$, we find that
\begin{align*}
	\kappa (\langle x, y \rangle) &= \sum_{j=0}^\infty \frac{\kappa^{(j)}(0)}{j!} \cdot \|x\|^j \cdot \|y\|^j \cdot \sum_{\ell=0}^j \mu_\ell^j \cdot P_d^{\ell}\left(\frac{\langle x, y \rangle}{\|x\| \cdot \|y\|}\right)\\
	&= \sum_{\ell=0}^\infty \left( \sum_{j=\ell}^\infty \mu_\ell^j \cdot \frac{\kappa^{(j)}(0)}{j!} \cdot \|x\|^j \cdot \|y\|^j \right) \cdot  P_d^{\ell}\left(\frac{\langle x, y \rangle}{\|x\| \cdot \|y\|}\right)\\
	&= \sum_{\ell=0}^\infty \left( \sum_{i=0}^\infty \mu_\ell^{\ell+2i} \cdot \frac{\kappa^{(\ell+2i)}(0)}{(\ell+2i)!} \cdot \|x\|^{\ell+2i} \cdot \|y\|^{\ell+2i} \right) \cdot  P_d^{\ell}\left(\frac{\langle x, y \rangle}{\|x\| \cdot \|y\|}\right)\\
	&= \sum_{\ell=0}^\infty \left( \sum_{i=0}^\infty \widetilde{h}_{\ell,i}( \|x\|) \cdot \widetilde{h}_{\ell,i}(\|y\|) \right) \cdot  P_d^{\ell}\left(\frac{\langle x, y \rangle}{\|x\| \cdot \|y\|}\right),
\end{align*}
where the functions $\widetilde{h}_{\ell,i}(\cdot)$ are defined as
\[ \widetilde{h}_{\ell,i}(t) \coloneqq \sqrt{\mu_\ell^{\ell+2i} \cdot \frac{\kappa^{(\ell+2i)}(0)}{(\ell+2i)!}} \cdot t^{\ell+2i}. \]
Note that since $\kappa(\cdot)$ is a valid positive semi-definite kernel function, it's derivatives $\kappa^{(\ell+2i)}(0)$ are all non-negative \cite{schoenberg1988positive}, thus the above function is real-valued.
Now by \cref{eq:gegenbauer-expansion-taylor-series-kappa}, the function $h_{i,\ell}(t)$ defined above satisfies
\begin{align*}
	\widetilde{h}_{\ell,i}(t) = \sqrt{ \frac{\alpha_{\ell,d}}{2^\ell}\cdot \frac{\Gamma(\frac{d}{2}) \cdot \kappa^{(\ell+2i)}(0)}{\sqrt{\pi} \cdot (2i)!} \cdot \frac{\Gamma(i+\frac{1}{2}) }{\Gamma(i+\ell+ \frac{ d}{2})} } \cdot t^{\ell+2i}.
\end{align*}
This completes the proof of \cref{lem-dot-prod-spherical-harmonic-expansion}. 
\end{proofof}

\section{Gaussian and Neural Tangent Kernels are GZK}\label{appndx-ntk-is-gzk}
In this section we show that the Gaussian and Neural Tangent Kernels are contained in the class of GZKs.

\begin{lemma}[Gaussian kernel is a GZK]\label{eq-gauss-kernel-zonal-expansion}
For any $x, y \in \RR^d$, any integer $d \ge 3$, the eigenfunction expansion of the Gaussian kernel can be written as,
	\begin{align*}
		e^{-\| x - y \|_2^2/2} \coloneqq \sum_{\ell=0}^\infty \left( \sum_{i=0}^\infty \widetilde{h}_{\ell,i}(\|x\|) \widetilde{h}_{\ell,i}(\|y\|) \right) P_d^{\ell}\left( \frac{\langle x, y \rangle}{\|x\| \|y\|}\right),
	\end{align*}
	where $\widetilde{h}_{\ell,i}(\cdot)$ are real-valued monomials defined as follows for integers $\ell,i \ge 0$ and any $t \in \RR$:
	\begin{align*}
	    \widetilde{h}_{\ell,i}(t) \coloneqq \sqrt{ \frac{\alpha_{\ell,d}}{2^\ell}  \frac{\Gamma(\frac{d}{2})}{\sqrt{\pi}(2i)!}   \frac{\Gamma(i+\frac{1}{2}) }{\Gamma(i+\ell+ \frac{ d}{2})} } \cdot t^{\ell+2i} \cdot e^{-t^2/2}.
	\end{align*}
\end{lemma}
\begin{proofof}{\cref{eq-gauss-kernel-zonal-expansion}}
First note that for the Gaussian kernel we can write, $k(x,y) = e^{-\norm{x-y}^2/2} = e^{-\|x\|^2/2} e^{-\|y\|^2/2} e^{\langle x, y \rangle}$. Applying \cref{lem-dot-prod-spherical-harmonic-expansion} to the exponential kernel function $e^{\langle x, y \rangle}$, we have
\begin{equation}
    e^{\langle x, y \rangle} \coloneqq \sum_{\ell=0}^\infty \left( \sum_{i=0}^\infty \widetilde{h}^{\exp}_{\ell,i}(\|x\|) \widetilde{h}^{\exp}_{\ell,i}(\|y\|) \right) P_d^{\ell}\left( \frac{\langle x, y \rangle}{\|x\| \|y\|}\right),
\end{equation}
where
\begin{align} 
	\widetilde{h}_{\ell,i}^{\exp}(t) = \sqrt{ \frac{\alpha_{\ell,d}}{2^\ell} \cdot \frac{\Gamma(\frac{d}{2})}{\sqrt{\pi}  (2i)!} \cdot \frac{\Gamma(i+\frac{1}{2}) }{\Gamma(i+\ell+ \frac{ d}{2})} } \cdot t^{\ell+2i}.
\end{align}
The reason for the above is because all derivatives of the exponential function are equal to $1$ at the origin.
So, using the above we have,

\begin{equation}\label{eq:gauss-kernel-zonal-expansion}
    e^{-\frac{\|x-y\|_2^2}{2}} = \sum_{\ell=0}^\infty \left( \sum_{i=0}^\infty e^{-\frac{\|x\|^2}{2}} \widetilde{h}_{\ell,i}^{\exp} (\|x\|) \cdot e^{-\frac{\|y\|^2}{2}} \widetilde{h}_{\ell,i}^{\exp} (\|y\|) \right) \cdot P_d^{\ell}\left( \frac{\langle x, y \rangle}{\|x\| \cdot \|y\|}\right).
\end{equation}
This shows that the Gaussian kernel can be represented in the form of 
	\begin{align*}
		e^{-\frac{\|x-y\|_2^2}{2}} = \sum_{\ell=0}^\infty \left( \sum_{i=0}^\infty \widetilde{h}_{\ell,i}(\|x\|) \widetilde{h}_{\ell,i}(\|y\|) \right) P_d^{\ell}\left( \frac{\langle x, y \rangle}{\|x\| \|y\|}\right),
	\end{align*}
with $\widetilde{h}_{\ell,i}(t) = e^{-t^2/2} \cdot  \widetilde{h}_{\ell,i}^{\exp}(t)$.
\end{proofof}
\cref{eq-gauss-kernel-zonal-expansion} shows that the Gaussian kernel is a GZK as per \cref{def-generalized-zonal-kernels} with \[h_\ell(t) = \left[ \sqrt{ \frac{\alpha_{\ell,d}}{2^\ell}  \frac{\Gamma(\frac{d}{2})}{\sqrt{\pi}(2i)!}   \frac{\Gamma(i+\frac{1}{2}) }{\Gamma(i+\ell+ \frac{ d}{2})} } \cdot t^{\ell+2i} \cdot e^{-t^2/2} \right]_{i=0}^\infty.\]

Next, we show that the Neural Tangent Kernel (NTK) of an infinitely wide network with ReLU activation is a GZK. It was shown in \citep[Definition~1]{zandieh2021scaling} that the depth-$L$ NTK with ReLU activation has the following normalized dot-product form,
\begin{equation}\label{eq:ntk-def}
    \Theta_{\tt ntk}^{(L)}(x,y) \coloneqq \|x\| \|y\| \cdot K_{\tt relu}^{(L)}\left( \frac{\langle x, y \rangle}{\|x\|_\|y\|} \right), \text{~~~for any } x,y \in \RR^d, 
\end{equation}
where $K_{\tt relu}^{(L)} : [-1,1] \to \RR$ is some smooth univariate function that can be computed using a recursive relation. We show that this kernel is indeed a GZK.

\begin{lemma}[Neural Tangent Kernel is a GZK]\label{eq-ntk-zonal-expansion}
For any $x, y \in \RR^d$, any integers $d \ge 3$ and $L\ge 1$, the eigenfunction expansion of the depth-$L$ NTK defined in \citep[Definition~1]{zandieh2021scaling} can be written as,
	\begin{align*}
		\Theta_{\tt ntk}^{(L)}(x,y) \coloneqq \sum_{\ell=0}^\infty \widetilde{h}_{\ell}(\|x\|) \widetilde{h}_{\ell}(\|y\|) \cdot P_d^{\ell}\left( \frac{\langle x, y \rangle}{\|x\| \|y\|}\right),
	\end{align*}
	where $\widetilde{h}_{\ell}(\cdot)$ are linear univariate functions defined as follows for integer $\ell \ge 0$ and any $t \in \RR$:
	\begin{align*}
	    \widetilde{h}_{\ell}(t) \coloneqq \sqrt{\alpha_{\ell, d} \cdot \frac{|\SS^{d-2}|}{|\SS^{d-1}|} \cdot \int_{-1}^1 K_{\tt relu}^{(L)}(\tau) P_d^{\ell}(\tau) (1-\tau^2)^{\frac{d-3}{2}} d\tau} \cdot t,
	\end{align*}
	where $K_{\tt relu}^{(L)} : [-1,1] \to \RR$ is the univariate function defined as per \citep[Definition~1]{zandieh2021scaling}.
\end{lemma}
\begin{proofof}{\cref{eq-ntk-zonal-expansion}}
We start by finding the Gegenbauer series expansion of $K_{\tt relu}^{(L)}(t)$ using \cref{eq-gegen-expansion} and \cref{eq-gegen-expansion-coeffs}:
\begin{align}
	K_{\tt relu}^{(L)}(t) = \sum_{\ell=0}^\infty c_\ell \cdot P_d^{\ell}(t),
\end{align}
where the Gegenbauer coefficients $c_\ell$, can be computed as follows,
\begin{align*} 
	c_\ell = \alpha_{\ell, d} \cdot \frac{|\SS^{d-2}|}{|\SS^{d-1}|} \cdot \int_{-1}^1 K_{\tt relu}^{(L)}(t) P_d^{\ell}(t) (1-t^2)^{\frac{d-3}{2}} dt.
\end{align*}
Therefore, using \cref{eq:ntk-def} we have,
\[ \Theta_{\tt ntk}^{(L)}(x,y) \coloneqq \sum_{\ell=0}^\infty c_\ell \cdot \|x\|_2 \|y\|_2 \cdot P_d^{\ell}\left( \frac{\langle x, y \rangle}{\|x\|_2 \|y\|_2} \right). \]
Therefore the lemma follows.
\end{proofof}

\section{Mercer Decomposition of GZK}
In this section we prove the lemmas about the Mercer decomposition of Zonal and Generalized Zonal kernels.

\subsection{Proof of \cref{lem-zonal-kernel-feature-map}} \label{sec-proof-zonal-kernel-feature-map}
\zonalkernelfeaturemap*

\begin{proofof}{\cref{lem-zonal-kernel-feature-map}}
We observe that
\begin{align*}
	\mathbb{E}_{w \sim \mathcal{U}(\SS^{d-1})}[\phi_x(w)\cdot\phi_y(w)] 
	&= \mathbb{E}_w \left[ \sum_{\ell, \ell'=0}^\infty \sqrt{c_\ell c_{\ell'} \alpha_{\ell,d} \alpha_{\ell', d}} \cdot P_d^{\ell}(\inner{x,w}) \cdot P_d^{\ell'}(\inner{y,w}) \right] \\
	&= \sum_{\ell, \ell'=0}^\infty \sqrt{c_\ell c_{\ell'} \alpha_{\ell,d} \alpha_{\ell', d}} \cdot \mathbb{E}_w \left[ P_d^{\ell}(\inner{x,w}) \cdot P_d^{\ell'}(\inner{y,w}) \right] \\
	&=  \sum_{\ell, \ell'=0}^\infty \sqrt{c_\ell c_{\ell'} \alpha_{\ell,d} \alpha_{\ell', d}} \cdot \frac{P_d^{\ell}(\inner{x,y})}{\alpha_{\ell,d}} \cdot  \mathbbm{1}_{\{\ell = \ell'\}} \\
	&= \sum_{\ell=0}^\infty c_\ell P_d^{\ell}(\inner{x,y}) = \kappa(\inner{x,y}). 
\end{align*}
where the third equality comes from \cref{lem-gegen-kernel-properties}. This completes the proof of \cref{lem-zonal-kernel-feature-map}.
\end{proofof}

\subsection{Proof of \cref{lem-gzk-feature-map}} \label{sec-proof-lem-gzk-feature-map}
In this section we prove that \cref{lem-gzk-feature-map} gives a Mercer decomposition of the GZK.

\featuremapgzk*
\begin{proofof}{\cref{lem-gzk-feature-map}}
By \cref{eq-featuremap-gen-zonal-kernls}, 
\begin{align*}
    \E_w\left[ \inner{\phi_x(w), \phi_y(w)}\right]
    &= \E_w\left[  \inner{\phi_x(w) , \phi_y(w)} \right]\\
    &= \E_w\left[ \inner{\sum_{\ell=0}^\infty \sqrt{\alpha_{\ell,d}} h_\ell(\norm{x}) P_d^\ell\left( \frac{\inner{x,w}}{\norm{x}}\right),  \sum_{\ell'=0}^\infty \sqrt{\alpha_{\ell',d}} h_{\ell'}(\norm{y}) P_d^{\ell'}\left( \frac{\inner{y,w}}{\norm{y}}\right)} \right]\\
    &= \sum_{\ell=0}^\infty \sum_{\ell'=0}^\infty  \sqrt{\alpha_{\ell,d} \cdot \alpha_{\ell',d}} \cdot \inner{h_{\ell}(\norm{x}), h_{\ell'} (\norm{y})} \cdot
    \E_w\left[ P_d^\ell\left( \frac{\inner{x,w}}{\norm{x}}\right) P_d^{\ell'}\left( \frac{\inner{y,w}}{\norm{y}}\right) \right]\\
    &= \sum_{\ell=0}^\infty \sum_{\ell'=0}^\infty  \sqrt{\alpha_{\ell,d} \cdot \alpha_{\ell',d}} \cdot \inner{h_{\ell}(\norm{x}), h_{\ell'} (\norm{y})} \cdot \frac{1}{\alpha_{\ell,d}} \cdot P_d^{\ell}\left( \frac{\langle x, y \rangle}{\|x\| \|y\|}\right) \cdot \mathbbm{1}_{\{\ell = \ell'\}}\\
	&= \sum_{\ell=0}^\infty \inner{h_{\ell} (\|x\|) , h_{\ell} (\|y\|)} \cdot P_d^{\ell}\left( \frac{\langle x, y \rangle}{\|x\| \|y\|}\right),
\end{align*}
where the second last line above follows from \cref{lem-gegen-kernel-properties}. 
This completes the proof of \cref{lem-gzk-feature-map}.
\end{proofof}

\section{Leverage Scores of the GZK Feature Operator} \label{appndx;leverage-score-upperbound}
In this section we prove the uniform upper bound on ridge leverage scores of the GZK feature operator $\BPhi$ defined in \cref{eq:def-feature-operator} as well as some other useful properties of the leverage scores.
We start by calculating the \emph{average} of the ridge leverage scores defined in \cref{defn:leverage-score}, a.k.a. \emph{statistical dimension} of the kernel matrix,
\begin{align*}
    s_\lambda &\coloneqq \E_{w \sim \mathcal{U}(\SS^{d-1})} \left[ \tau_\lambda(w) \right]\\
    &= \E_{w \sim \mathcal{U}(\SS^{d-1})} \left[ \trace \left(\Phi_w^\top \cdot \left( \BPhi^* \BPhi + \lambda \I \right)^{-1} \cdot \Phi_w \right) \right]\\
    &= \trace \left( \left( \BPhi^* \BPhi + \lambda \I \right)^{-1} \cdot \E_{w \sim \mathcal{U}(\SS^{d-1})} \left[ \Phi_w \Phi_w^\top \right] \right)\\
    &= \trace \left( \left( \BPhi^* \BPhi + \lambda \I \right)^{-1} \cdot \BPhi^* \BPhi \right)\\
    &= \trace \left( \left( \K + \lambda \I \right)^{-1} \cdot \K \right).
\end{align*}

Next, we use the fact that the ridge leverage scores can be characterized in terms of a least-squares minimization problem, which is crucial for approximately computing the leverage scores distribution. This fact was previously exploited in \cite{avron2017random}.

\begin{lemma}[Minimization characterization of ridge leverage scores]\label{lem:min-char-ridge-leverage}
For any $\lambda > 0$, let $\BPhi$ be the operator defined in \cref{eq:def-feature-operator}, and its leverage score $\tau_\lambda(\cdot)$ be defined as in \cref{defn:leverage-score}. 
If we let $\Phi_w^i$ denote the $i^{th}$ column of the matrix $\Phi_w \in \Rbb^{n \times s}$ defined in \cref{eq:row-feature-operator} for any $i \in [s]$, the following holds,
\begin{equation}\label{eq:leverage-score-min-char}
    \tau_\lambda(w) = \sum_{i \in [s]} \left( \min_{g_i \in L^2(\SS^{d-1}, \RR^s )} \|g_i\|_{L^2(\SS^{d-1}, \RR^s)}^2 + \lambda^{-1} \cdot \left\| \BPhi^* g_i - \Phi_w^i \right\|_2^2 \right)  ~~~~~ \text{ for }w \in \SS^{d-1}.
\end{equation}

\end{lemma}
We remark that this lemma is in fact a modification and generalization of Lemma~11 of \cite{avron2017random}. We prove this lemma here for the sake of completeness.

\begin{proofof}{\cref{lem:min-char-ridge-leverage}}
For any $i \in [s]$ let $g_i^*$ denote the least-squares solution to the $i^{th}$ summand in right hand side of \cref{eq:leverage-score-min-char}. The optimal solution $g_i^*$ can be obtained from the normal equation as follows,
\begin{align*}
    g_i^* = \left( \BPhi \BPhi^* + \lambda \I_{L^2(\SS^{d-1}, \RR^s)} \right)^{-1} \cdot \BPhi \cdot \Phi_w^i = \BPhi \cdot \left( \BPhi^* \BPhi  + \lambda \I_n \right)^{-1} \cdot \Phi_w^i = \BPhi \cdot \left( \K  + \lambda \I_n \right)^{-1} \cdot \Phi_w^i,
\end{align*}
    where the second equality above follows from  the matrix inversion lemma for operators \cite{ogawa1988operator}.
    We now have,
    \begin{align*}
        \|g_i^*\|_{L^2(\SS^{d-1}, \RR^s) }^2 &= \left< \BPhi \cdot \left( \K + \lambda \I_n \right)^{-1} \cdot \Phi_w^i, \BPhi \cdot \left( \K + \lambda \I_n \right)^{-1} \cdot \Phi_w^i \right>_{L^2(\SS^{d-1}, \RR^s ) }\\
        &= \left< \Phi_w^i , \left( \BPhi \cdot \left( \K + \lambda \I_n \right)^{-1} \right)^* \cdot \BPhi \cdot \left( \K + \lambda \I_n \right)^{-1} \cdot \Phi_w^i \right>\\
        &= \left< \Phi_w^i, \left( \K + \lambda \I_n \right)^{-1} \cdot \K \cdot \left( \K + \lambda \I_n \right)^{-1} \cdot \Phi_w^i \right>\\
        &= \Phi_w^{i \top} \cdot \left( \K + \lambda \I \right)^{-1} \cdot \Phi_w^i - \lambda \cdot \Phi_w^{i \top} \cdot \left( \K + \lambda \I \right)^{-2} \cdot \Phi_w^i.
    \end{align*}
    We also have,
    \begin{align*}
        \left\| \BPhi^* g_i^* - \Phi_w^i \right\|_2^2 
        &= \left\| \BPhi^* \BPhi \cdot \left( \K  + \lambda \I_n \right)^{-1} \cdot \Phi_w^i - \Phi_w^i \right\|_2^2\\
        &= \left\| -\lambda \left( \K  + \lambda \I_n \right)^{-1} \cdot \Phi_w^i \right\|_2^2\\
        &=\lambda^2 \cdot \Phi_w^{i \top} \cdot \left( \K  + \lambda \I_n \right)^{-2} \cdot \Phi_w^i.
    \end{align*}
    Now by combining these equalities we have,
    \[ \|g_i^*\|_{L^2(\SS^{d-1}, \RR^s)}^2 + \lambda^{-1} \cdot \left\| \BPhi^* g_i^* - \Phi_w^i \right\|_2^2 = \Phi_w^{i \top} \cdot \left( \K + \lambda \I \right)^{-1} \cdot \Phi_w^i. \]
    Now summing the above over all $i \in [s]$ gives the lemma,
    \begin{align*}
        \sum_{i \in [s]}\|g_i^*\|_{L^2(\SS^{d-1}, \RR^s)}^2 + \lambda^{-1} \cdot \left\| \BPhi^* g_i^* - \Phi_w^i \right\|_2^2 &= \sum_{i \in [s]} \Phi_w^{i \top} \cdot \left( \K + \lambda \I \right)^{-1} \cdot \Phi_w^i\\
        &= \trace \left( \Phi_w^\top \cdot \left( \K + \lambda \I \right)^{-1} \cdot \Phi_w \right)\\
        &\coloneqq \tau_\lambda(w).
    \end{align*}
\end{proofof}

Now using the minimization characterization of the leverage score we can prove a uniform upper bound for any GZK and its corresponding feature as follows,

\leveragescoreupperbound*

\begin{proofof}{\cref{lem:leverage-score-upper-bound}}
We prove the lemma using min-characterization of ridge leverage scores.
Let $\mu \coloneqq \frac{6 \lambda s}{\pi^2 n}$ and define the data-dependent quantities $R_{\ell}$ as follows:
\[ R_{\ell} \coloneqq  \frac{(\ell + 1)^2}{n} \cdot \sum_{j \in [n]} \left\| h_{\ell}(\|x_j\|) \right\|^2 ~,~~~~~~~\text{ for } \ell =0,1,2, \ldots \]
Now, for any $i \in [s]$, let us define the function $g^i_w \in L^2(\SS^{d-1}, \RR^s)$ as,
\[ g^i_w(\sigma) \coloneqq \left( \sum_{\ell=0}^\infty \alpha_{\ell,d} \cdot \mathbbm{1}_{ \{ R_{\ell} \ge \mu \} } \cdot P_{d}^{\ell}\left( \langle \sigma , w \rangle \right) \right) \cdot e_i, \]
where $e_i \in \RR^s$ is the standard basis vector along the $i^{th}$ coordinate.
For this function we have,
\begin{align*}
    \|g^i_w\|_{L^2(\SS^{d-1}, \RR^s)}^2 &= \left\| \sum_{\ell=0}^\infty \alpha_{\ell,d} \cdot \mathbbm{1}_{ \{ R_{\ell} \ge \mu \} } \cdot P_{d}^{\ell}\left( \langle \cdot , w \rangle \right) \right\|_{L^2(\SS^{d-1})}^2 \\
    &= \sum_{\ell=0}^\infty \sum_{\ell'=0}^\infty \alpha_{\ell,d} \alpha_{\ell',d} \cdot \mathbbm{1}_{ \{ R_{\ell} \ge \mu \} } \cdot \mathbbm{1}_{ \{ R_{\ell'} \ge \mu \} } \cdot \E_{\sigma \sim \mathcal{U}(\SS^{d-1})} \left[ P_{d}^{\ell}\left( \langle \sigma , w \rangle \right) \cdot P_{d}^{\ell'}\left( \langle \sigma , w \rangle \right) \right]\\
    &= \sum_{\ell=0}^\infty \alpha_{\ell,d} \cdot\mathbbm{1}_{ \{ R_{\ell} \ge \mu \} } \cdot P_{d}^{\ell}\left( \langle w , w \rangle \right) = \sum_{\ell=0}^\infty \alpha_{\ell,d} \cdot \mathbbm{1}_{ \{ R_{\ell} \ge \mu \} },
\end{align*}
where the second line above follows from the definition of norm in the Hilbert space $L^2(\SS^{d-1}, \RR)$ and the third line follows from \cref{lem-gegen-kernel-properties} together with the fact that $P_{d}^{\ell}\left( \langle w , w \rangle \right) = P_{d}^{\ell}(1)=1$. Thus, by summing the above over all $i \in [s]$ we get the following,
\begin{equation}\label{eq:lev-score-upper-bound-norm-test-function}
    \sum_{i \in [s]} \|g^i_w\|_{L^2(\SS^{d-1}, \RR^s)}^2 = s \cdot \sum_{\ell=0}^\infty \alpha_{\ell,d} \cdot \mathbbm{1}_{ \{ R_{\ell} \ge \mu \} }
\end{equation}

Furthermore, for any $j \in [n]$ we have,
\begin{align*}
    [\BPhi^* g^i_w]_j &= \langle \phi_{x_j}, g^i_w \rangle_{L^2(\SS^{d-1}, \RR^s)}\\
    &= \left< \sum_{\ell=0}^\infty \sqrt{\alpha_{\ell,d}} \cdot \left[h_{\ell}(\|x_j\|)\right]_i \cdot P_{d}^{\ell}\left( \frac{\langle x_j, \cdot \rangle}{\|x_j\|}\right) , \sum_{\ell=0}^\infty \alpha_{\ell,d} \cdot \mathbbm{1}_{ \{ R_{\ell} \ge \mu \} } \cdot P_{d}^{\ell}\left( \langle \cdot , w \rangle\right) \right>_{L^2(\SS^{d-1})}\\
    &= \sum_{\ell=0}^\infty \sqrt{\alpha_{\ell,d}} \cdot \left[h_{\ell}(\|x_j\|)\right]_i \cdot \mathbbm{1}_{ \{ R_{\ell} \ge \mu \} } \cdot P_{d}^{\ell}\left( \frac{\langle x_j, w \rangle}{\|x_j\|}\right),
\end{align*}
where the third line above follows from  \cref{lem-gegen-kernel-properties}. Using the above equality along with definition of $\Phi_w$ in \cref{eq:row-feature-operator} and noting that $\Phi_w^i$ is the $i^{th}$ column of this matrix, we can write,
\begin{align*}
    \left\| \BPhi^* g^i_w - \Phi_w^i \right\|_2^2 &= \sum_{j=1}^n \left| [\BPhi^* g^i_w]_j - [\Phi_w]_{j,i} \right|^2\\
    &= \sum_{j=1}^n \left| \langle \phi_{x_j}, g^i_w \rangle_{L^2(\SS^{d-1}, \RR^s)} - \left[\phi_{x_j}(w)\right]_i \right|^2\\
    &= \sum_{j=1}^n \left| \sum_{\ell=0}^\infty \sqrt{\alpha_{\ell,d}} \cdot \left[h_{\ell}(\|x_j\|)\right]_i \cdot \mathbbm{1}_{ \{ R_{\ell} < \mu \} } \cdot P_{d}^{\ell}\left( \frac{\langle x_j, w \rangle}{\|x_j\|}\right) \right|^2\\
    &\le \sum_{j=1}^n \left( \sum_{\ell=0}^\infty \sqrt{\alpha_{\ell,d}} \cdot \left| \left[h_{\ell}(\|x_j\|)\right]_i \right| \cdot \mathbbm{1}_{ \{ R_{\ell} < \mu \} } \right)^2\\
    &= \sum_{j=1}^n \left( \sum_{\ell=0}^\infty \sqrt{\alpha_{\ell,d}} \cdot \sqrt{R_{\ell}} \cdot \mathbbm{1}_{ \{ R_{\ell} < \mu \} } \cdot \frac{|\left[h_{\ell}(\|x_j\|)\right](i)| }{\sqrt{R_{\ell}}} \right)^2\\
    &\le \sum_{j=1}^n \left( \sum_{\ell=0}^\infty \alpha_{\ell,d} \cdot R_{\ell} \cdot \mathbbm{1}_{ \{ R_{\ell} < \mu \} } \right) \cdot \left( \sum_{\ell=0}^\infty \frac{|\left[h_{\ell}(\|x_j\|)\right](i)|^2 \cdot \mathbbm{1}_{ \{ 0 < R_{\ell} < \mu \} } }{R_{\ell}} \right) \\
    &= \left( \sum_{\ell=0}^\infty \alpha_{\ell,d} \cdot R_{\ell} \cdot \mathbbm{1}_{ \{ R_{\ell} < \mu \} } \right) \cdot \sum_{\ell=0}^\infty \frac{\sum_{j=1}^n  |\left[h_{\ell}(\|x_j\|)\right](i)|^2 \cdot \mathbbm{1}_{ \{ 0 < R_{\ell} < \mu \} } }{R_{\ell}} ,
\end{align*}
where the first inequality above follows from the fact that $\left| P_{d}^{\ell}( t ) \right| \le 1$ for $t \in [-1,1]$ (See Equation (2.116) in~\cite{atkinson2012spherical}) and the second inequality comes from Cauchy–Schwarz inequality. Therefore, if we sum the above over all $i \in [s]$ we find the following inequlity,
\begin{align*}
    \sum_{i \in [s]} \left\| \BPhi^* g^i_w - \Phi_w^i \right\|_2^2 & \le \left( \sum_{\ell=0}^\infty \alpha_{\ell,d} \cdot R_{\ell} \cdot \mathbbm{1}_{ \{ R_{\ell} < \mu \} } \right) \cdot \sum_{\ell=0}^\infty \frac{\sum_{j=1}^n  \left\|h_{\ell}(\|x_j\|)\right\|^2 \cdot \mathbbm{1}_{ \{ 0 < R_{\ell} < \mu \} } }{R_{\ell}} \\
    &\le \frac{ \pi^2 n}{6} \cdot \sum_{\ell=0}^\infty \alpha_{\ell,d} \cdot R_{\ell} \cdot \mathbbm{1}_{ \{ R_{\ell} < \mu \} },
\end{align*}
where the last line above follows from the definition of $R_{\ell}$.
Therefore, by combining the above with the norm of $g^i_w$'s in \cref{eq:lev-score-upper-bound-norm-test-function}, we find that,
\begin{align*}
    \sum_{i \in [s]} \|g^i_w\|_{L^2(\SS^{d-1}, \RR^s)}^2 + \lambda^{-1} \cdot \left\| \BPhi^* g^i_w - \Phi_w^i \right\|_2^2 &\le s \cdot \sum_{\ell=0}^\infty \alpha_{\ell,d} \cdot \mathbbm{1}_{ \{ R_{\ell} \ge \mu\} } + \frac{ \pi^2 n}{6 \lambda} \cdot  \sum_{\ell=0}^\infty \alpha_{\ell,d} \cdot R_{\ell} \cdot \mathbbm{1}_{ \{ R_{\ell} < \mu \} } \\
    &\le \sum_{\ell=0}^\infty \alpha_{\ell,d} \cdot \left( s \cdot \mathbbm{1}_{ \{ R_{\ell} \ge \mu\} } + s\mu^{-1} R_\ell \cdot \mathbbm{1}_{ \{ R_{\ell} < \mu \} } \right)\\
    &\le \sum_{\ell=0}^\infty \alpha_{\ell,d} \cdot \min \left\{ s\mu^{-1} R_{\ell}, s \right\}.
\end{align*}
Plugging in the values of $R_{\ell}$ proves the lemma, because by \cref{lem:min-char-ridge-leverage}, $\tau_\lambda(w) \le \sum_{i \in [s]} \|g^i_w\|_{L^2(\SS^{d-1}, \RR^s)}^2 + \lambda^{-1} \cdot \left\| \BPhi^* g^i_w - \Phi_w^i \right\|_2^2$ for any $w \in \SS^{d-1}$.
\end{proofof}

\section{Spectral Approximation to GZK Kernel Matrix}\label{sec:proof-thm-spectral}
We will use the following version of the matrix Bernstein inequality to show spectral guarantees for our leverage scores sampling method.

\begin{lemma}[Restatement of Corollary 7.3.3 of \cite{tropp2015introduction}]\label{bernstein-matrix}
	Let $\B$ be a fixed $n \times n$ matrix. Construct an $n \times n$ matrix $\R$ that, almost surely, satisfies,
	$$\E [\R]=\B \text{~~~and~~~} \|\R\|_{\op} \le L.$$
	Let $\M_1$ and $\M_2$ be semi-definite upper bounds for the expected squares,
	\[ \E [\R\R^{*}] \preceq \M_1, ~~~~\text{ and }~ \E [\R^{*}\R]\preceq \M_2 \]
	Define the quantities $M=\max\{\| \M_1 \|_{\op}, \|\M_2 \|_{\op}\}$.
	Form the matrix sampling estimator,
	$$\bar{\R} = \frac{1}{m} \sum_{j=1}^m \R_j,$$
	where each $\R_j$ is an independent copy of $\R$. Then,
	$$\Pr \left[ \|\bar{\R} - \B\|_{\op} \ge \varepsilon \right] \le 4\cdot \frac{\trace(\M_1 + \M_2 )}{M} \cdot \exp\left( \frac{-m\varepsilon^2/2}{M + 2L\varepsilon/3} \right).$$
\end{lemma}

Now we can prove \cref{thm:main-spectral-approx}.
Our proof is a generalized version of Lemma~6 in \cite{avron2017random}. We prove this theorem here for the sake of completeness.

\specapproxgeneralizedzonal*
\begin{proofof}{\cref{thm:main-spectral-approx}}
Let $\K + \lambda \I = \V^\top \Ssigma^2 \V$ be the singular value decomposition of the kernel matrix $\K + \lambda \I$. It is sufficient to show that,
\[ \Pr \left[ \norm{ \Ssigma^{-1} \V \cdot \Z^\top \Z \cdot \V^\top \Ssigma^{-1} - \Ssigma^{-1} \V \cdot \K \cdot \V^\top \Ssigma^{-1}  }_{\op} \le \varepsilon \right] \ge 1 - \delta.\]

Now note that from definition of our random features matrix $\Z$ in \cref{def-random-feature-construction} we have,
\[ \Ssigma^{-1} \V \cdot \Z^\top \Z \cdot \V^\top \Ssigma^{-1} = \frac{1}{m} \cdot \sum_{j=1}^m \Ssigma^{-1} \V \cdot \Phi_{w_j} \Phi_{w_j}^\top \cdot \V^\top \Ssigma^{-1}. \]
Thus, because in \cref{def-random-feature-construction}, $w_j$'s are sampled independently from each other from the distribution $\mathcal{U}(\SS^{d-1})$, we can invoke \cref{bernstein-matrix} with the following arguments,
\[ \B \coloneqq \Ssigma^{-1} \V \cdot \K \cdot \V^\top \Ssigma^{-1}, ~~~~ \text{ and }~~ \R_j \coloneqq  \Ssigma^{-1} \V \cdot \Phi_{w_j} \Phi_{w_j}^\top \cdot \V^\top \Ssigma^{-1}. \]
Now we verify that the preconditions of \cref{bernstein-matrix} holds. First note that $\E[\R_j] =\Ssigma^{-1} \V \cdot \E_{w_j \sim \mathcal{U}(\SS^{d-1})} [\Phi_{w_j} \Phi_{w_j}^\top ]\cdot \V^\top \Ssigma^{-1} = \B$. Now we need to bound the operator norm of $\R_j$ and the stable rank $\E[\R_j^2]$. Using the cyclic property of trace, we can upper bound the operator norm $\|R_j\|_{\op}$ as follows,
\begin{align*}
    \|\R_j\|_{\op} &\le \trace(\R_j) \\
    &= \trace \left( \Ssigma^{-1} \V \cdot \Phi_{w_j} \Phi_{w_j}^\top \cdot \V^\top \Ssigma^{-1} \right)\\
    &= \trace \left(  \Phi_{w_j}^\top \cdot \V^\top \Ssigma^{-2} \V \cdot \Phi_{w_j} \right)\\
    &= \trace \left(  \Phi_{w_j}^\top \cdot \left( \K + \lambda \I \right)^{-1} \cdot \Phi_{w_j} \right)\\
    &= \tau_\lambda(w_j),
\end{align*}
where the last line above follows from \cref{defn:leverage-score}.
This implies the following for any $j$,
\begin{equation}\label{operator-norm-bound-matrix-bernstein} \|\R_j\|_{\op} \le \max_{w \in \SS^{d-1}}\tau_\lambda(w) \coloneqq L. \end{equation}

We also have,
\begin{align*}
    \R_j^2 &= \Ssigma^{-1} \V \cdot \Phi_{w_j} \Phi_{w_j}^\top \cdot \V^\top \Ssigma^{-1} \cdot \Ssigma^{-1} \V \cdot \Phi_{w_j} \Phi_{w_j}^\top \cdot \V^\top \Ssigma^{-1}\\
    &= \Ssigma^{-1} \V \cdot \Phi_{\sigma_j} \Phi_{w_j}^\top \cdot \left( \K + \lambda \I \right)^{-1} \cdot \Phi_{w_j} \Phi_{w_j}^\top \cdot \V^\top \Ssigma^{-1}\\
    &\preceq \trace \left( \Phi_{w_j}^\top \cdot \left( \K + \lambda \I \right)^{-1} \cdot \Phi_{w_j} \right) \cdot \Ssigma^{-1} \V \cdot \Phi_{w_j} \Phi_{w_j}^\top \cdot \V^\top \Ssigma^{-1}\\
    &= \tau_\lambda(w_j) \cdot \Ssigma^{-1} \V \cdot \Phi_{w_j} \Phi_{w_j}^\top \cdot \V^\top \Ssigma^{-1}.
\end{align*}
Now if we let $\lambda_1 \ge \lambda_2 \ge \ldots \ge \lambda_n$ be the eigenvalues of the kernel matrix $\K$ we find that the following holds for any $j$,
\begin{align}
    \E \left[ \R_j^2 \right] &\preceq \left( \max_{w \in \SS^{d-1}}\tau_\lambda(w) \right)\cdot \E_{w_j \sim \mathcal{U}(\SS^{d-1})} \left[ \Ssigma^{-1} \V \cdot \Phi_{w_j} \Phi_{w_j}^\top \cdot \V^\top \Ssigma^{-1} \right] \nonumber\\
    &= L \cdot \Ssigma^{-1} \V \cdot \K \cdot \V^\top \Ssigma^{-1}\nonumber\\
    &=L \cdot \left( \I - \lambda \Ssigma^{-2} \right)\nonumber\\
    &= L \cdot \diag\left( \frac{\lambda_1}{\lambda_1 + \lambda}, \frac{\lambda_2}{\lambda_2 + \lambda}, \ldots \frac{\lambda_n}{\lambda_n + \lambda} \right) \coloneqq \D,
\end{align}
where $L$ is the operator norm upper bound defined in \cref{operator-norm-bound-matrix-bernstein}. Now by invoking \cref{lem:leverage-score-upper-bound}, we have the following upper bound,
\[ L = \max_{w \in \SS^{d-1}}\tau_\lambda(w) \le \sum_{\ell=0}^\infty \alpha_{\ell,d} \min \left\{ \frac{\pi^2 (\ell+1)^2}{6\lambda} \sum_{j \in [n]} \left\| h_{\ell}(\|x_j\|) \right\|^2 , s \right\}. \]
Therefore, by \cref{bernstein-matrix},
\begin{align*}
    \Pr \left[ \norm{ \frac{1}{m} \sum_{j=1}^m \R_j - \Ssigma^{-1} \V \cdot \K \cdot \V^\top \Ssigma^{-1} }_{\op} \ge \varepsilon \right] &\le 8 \cdot \frac{\trace(\D )}{ \norm{\D}_{\op} } \cdot \exp\left( \frac{-m\varepsilon^2/2}{\norm{\D}_{\op} + 2L\varepsilon/3} \right)\\
    &\le 8 \cdot \frac{s_\lambda}{\lambda_1 / (\lambda_1 + \lambda)} \cdot \exp\left( \frac{-m\varepsilon^2/2}{L + 2L\varepsilon/3} \right)\\
    &\le \delta,
\end{align*}
where the last line above is due to the fact that $\lambda_1 = \norm{\K}_{\op} \ge \lambda$ along with the value of $m$.
\end{proofof}

\section{Projection Cost Preserving Samples for GZK}\label{sec:proof-thm-project-cost}
In this section we show that our random features results in an approximate kernel matrix that satisfies the projection-cost preservation condition. 
This property ensures that it is possible to extract a near optimal low-rank approximation from the random features. The proof of our result is based on \cite{cohen2017input} which showed that unbiased leverage score sampling is sufficient for achieving this guarantee in discrete matrices. We extend this proof to the GZK quasi-matrix $\BPhi$.

\projcostpreservapproxgeneralizedzonal*
\begin{proofof}{\cref{thm:main-proj-cost-preserv-approx}}
The proof is nearly identical to the proof of Theorem~6 in \cite{cohen2017input} which proves that unbiased leverage score sampling results in projection-cost preserving samples in discrete matrices. We adopt the proof of Theorem~6 of \cite{cohen2017input} to our continuous operator $\BPhi$.
First, for ease of notation let $\Y\coloneqq \I - \P$. Now, note that we have the following,
\begin{align*} &\trace (\K - \P \K \P) = \trace (\Y \K \Y) ,\\
&\trace \left( \Z^\top \Z - \P \Z^\top \Z \P \right) = \trace \left( \Y \Z^\top \Z \Y \right). \end{align*}
So it is enough to show that,
\begin{align*}
    \frac{\trace (\Y \K \Y)}{1 + \varepsilon}
    \leq
    \trace \left( \Y \Z^\top \Z \Y \right)
    \leq
    \frac{\trace (\Y \K \Y )}{1 - \varepsilon}.
\end{align*}
Let $t$ be the index of the smallest eigenvalue of $\K$ such that $\lambda_t \ge \frac{1}{r}\sum_{i=r+1}^n \lambda_i = \lambda$. Let $\Q_t$ denote the projection onto the eigenspace of matrix $\K$ corresponding to $\lambda_1, \lambda_2, \ldots, \lambda_t$. Also let $\Q_{\setminus t}\coloneqq \I - \Q_t$.
We split,
\begin{equation}\label{eq:split-trace-K}
    \trace (\Y \K \Y ) = \trace (\Y \Q_t \K \Q_t \Y) + \trace (\Y \Q_{\setminus t} \K \Q_{\setminus t} \Y)
\end{equation}
Additionally, we split:
\begin{equation}\label{eq:split-trace-rand-feat}
    \trace (\Y \Z^\top \Z \Y ) = \trace (\Y \Q_t \Z^\top \Z \Q_t \Y) + \trace (\Y \Q_{\setminus t} \Z^\top \Z \Q_{\setminus t} \Y) + 2 \trace (\Y \Q_t \Z^\top \Z \Q_{\setminus t} \Y).
\end{equation}
\paragraph{Head Terms.} We first bound the term $\trace (\Y \Q_t \Z^\top \Z \Q_t \Y) - \trace (\Y \Q_t \K \Q_t \Y)$. First note that by \cref{eq:spectral-approx-thm}, for any vector $v \in \RR^n$ we have,
\[ (1 - \varepsilon)v^\top \Q_t \Z^\top \Z\Q_t v - \varepsilon \lambda \|\Q_t v\|_2^2
\le
v^\top \Q_t \K \Q_t v
\le
(1 + \varepsilon)v^\top \Q_t \Z^\top \Z\Q_t v + \varepsilon \lambda \|\Q_t v\|_2^2. \]
By definition of $t$, $\Q_t v$ is orthogonal to all eigenvectors of $\K$ except those with eigenvalue greater than or equal to $\lambda$. Thus,
\[ v^\top \Q_t \K \Q_t v \ge \lambda \| \Q_t  v\|_2^2. \]
This inequality combines with the previous equality to give,
\[ \frac{1 - \varepsilon}{1 + \varepsilon} \cdot v^\top \Q_t \Z^\top \Z\Q_t v
\le
v^\top \Q_t \K \Q_t v
\le
\frac{1 + \varepsilon}{1 - \varepsilon} \cdot v^\top \Q_t \Z^\top \Z\Q_t v, \]
for all $v \in \Rbb^n$. This implies that,
\begin{equation}\label{eq:head-spectral-approx}
\frac{1 - \varepsilon}{1 + \varepsilon} \cdot \Q_t \Z^\top \Z\Q_t
\preceq
 \Q_t \K \Q_t
\preceq
\frac{1 + \varepsilon}{1 - \varepsilon} \cdot \Q_t \Z^\top \Z\Q_t . \end{equation}
Using the above we conclude that,
\[ (1 - 3\varepsilon) \cdot \trace (\Y \Q_t \Z^\top \Z \Q_t \Y)
\le
 \trace (\Y \Q_t \K \Q_t \Y)
\le
(1 + 3\varepsilon) \cdot \trace (\Y \Q_t \Z^\top \Z \Q_t \Y) . \]

\paragraph{Tail Terms.} For the lower singular vectors of $\K$, \cref{thm:main-spectral-approx} does not give a multiplicative bound, so we do things a bit differently. Specifically, we start by writing:
\begin{align*}
    &\trace (\Y \Q_{\setminus t} \K \Q_{\setminus t} \Y) = \trace ( \Q_{\setminus t} \K \Q_{\setminus t}) - \trace (\P \Q_{\setminus t} \K \Q_{\setminus t} \P),\\
    &\trace (\Y \Q_{\setminus t} \Z^\top \Z \Q_{\setminus t} \Y) = \trace ( \Q_{\setminus t} \Z^\top \Z \Q_{\setminus t}) - \trace (\P \Q_{\setminus t} \Z^\top \Z \Q_{\setminus t} \P)
\end{align*}

We handle $\trace ( \Q_{\setminus t} \K \Q_{\setminus t})$ and $\trace ( \Q_{\setminus t} \Z^\top \Z \Q_{\setminus t})$ first. Since $\Z$ is constructed
via an unbiased sampling of $\BPhi$ rows, $\E[\Q_{\setminus t} \Z^\top \Z \Q_{\setminus t}] = \Q_{\setminus t} \K \Q_{\setminus t}$ and a scalar-version Chernoff bound is sufficient for showing that this value concentrates around its expectation. We have the following bound:
\[ \left| \trace ( \Q_{\setminus t} \Z^\top \Z \Q_{\setminus t}) - \trace ( \Q_{\setminus t} \K \Q_{\setminus t}) \right| \le \varepsilon r \lambda. \]
Note that the above inequality does not depend on the choice of projection $\P$, so it holds simultaneously for all $\P$. We do not provide more details on why the above inequality holds but it follows fairly straightforwardly from scalar Chernoff bound. For example, one can find a detailed proof in Lemma~20 of \cite{cohen2017input}.

Next, we compare $\trace (\P \Q_{\setminus t} \Z^\top \Z \Q_{\setminus t} \P)$ to $\trace (\P \Q_{\setminus t} \K \Q_{\setminus t} \P)$. We first claim that:
\begin{equation}\label{eq:tail-spectral-approx}
\Q_{\setminus t} \Z^\top \Z\Q_{\setminus t} - 3\varepsilon\lambda \I
\preceq
 \Q_{\setminus t} \K \Q_{\setminus t}
\preceq
\Q_{\setminus t} \Z^\top \Z\Q_{\setminus t} + 3\varepsilon\lambda \I. \end{equation}
The argument is similar to the one for \cref{eq:head-spectral-approx}.
Now, since $\P$ is a rank $r$ projection matrix this inequality implies that,
\[ \trace (\P \Q_{\setminus t} \Z^\top \Z\Q_{\setminus t} \P) - 3\varepsilon r \lambda
\le
 \trace(\P \Q_{\setminus t} \K \Q_{\setminus t} \P )
\le
\trace (\P \Q_{\setminus t} \Z^\top \Z\Q_{\setminus t} \P) + 3\varepsilon r \lambda \]
which combines with the previous bound to give the final bound:
\[ \left| \trace ( \Y \Q_{\setminus t} \Z^\top \Z\Q_{\setminus t} \Y ) - 
 \trace (\Y \Q_{\setminus t} \K \Q_{\setminus t} \Y) \right|
\le 4\varepsilon r \lambda. \]

\paragraph{Cross Term.}
Finally, we handle the cross term $2 \trace (\Y \Q_{m} \Z^\top \Z \Q_{\setminus t} \Y)$. We just need to show that it is small. To do so, we rewrite:
\begin{equation}
    \trace (\Y \Q_t \Z^\top \Z \Q_{\setminus t} \Y) =  \trace (\Y \K \K^{\dagger} \Q_t \Z^\top \Z \Q_{\setminus t} ),
\end{equation}
which holds since the columns of $\Q_t \Z^\top \Z \Q_{\setminus t}$ fall in the span of $\K$'s columns and the trailing $\Y$ gets eliminated by cyclic property of the trace. Now let us define the semi-inner product of matrices $\langle \M, \N \rangle \coloneqq \trace (\M \K^{\dagger}\N^\top )$. Thus, by Cauchy-Schwarz inequality, if we let $\K = \U \BSigma^2 \U^\top$ be the singular value decomposition of $\K$, we have,
\begin{align}
    \trace (\Y \K \K^{\dagger} \Q_{m} \Z^\top \Z \Q_{\setminus t} ) &\le \sqrt{\trace (\Y \K \K^{\dagger} \K \Y) \cdot \trace (\Q_{\setminus t} \Z^\top \Z \Q_t \K^{\dagger} \Q_t \Z^\top \Z \Q_{\setminus t} )} \nonumber\\
    &= \sqrt{\trace (\Y \K \Y) \cdot \trace (\Q_{\setminus t} \Z^\top \Z \U_t \BSigma^{-2}_t \U_t^\top \Z^\top \Z \Q_{\setminus t} )} \nonumber\\
    &= \sqrt{\trace (\Y \K \Y) \cdot \| \BSigma^{-1}_t \U_t^\top \Z^\top \Z \Q_{\setminus t} \|_F^2 }.\label{eq:cross-term-cauchy-schwarz}
\end{align}
To bound the second term, we write,
\[ \left\| \BSigma^{-1}_t \U_t^\top \Z^\top \Z \Q_{\setminus t} \right\|_F^2 = \sum_{i=1}^t \lambda_i^{-1} \cdot \|\Q_{\setminus t} \Z^\top \Z u_i\|_2^2, \]
where $u_i$ is the $i^{th}$ column of $\U$. Now we show that the summand is small for every $i \in [m]$.
Let vector $q_i$ be defined as $q_i\coloneqq \frac{\Q_{\setminus t} \Z^\top \Z u_i}{\|\Q_{\setminus t} \Z^\top \Z u_i\|_2}$. Then we have,
\begin{equation}\label{cross-term-norm-sum}
\left\| \BSigma^{-1}_t \U_t^\top \Z^\top \Z \Q_{\setminus t} \right\|_F^2 = \sum_{i=1}^t \lambda_i^{-1} \cdot \left( q_i^\top \Z^\top \Z u_i \right)^2.
\end{equation} 
Now, let us define the vector $v \coloneqq \frac{u_i}{\sqrt{\lambda_i}} + \frac{q_i}{\sqrt{\lambda}}$. Using \cref{eq:spectral-approx-thm} we can write,
\[(1 - \varepsilon)v^\top \Z^\top \Z v - \varepsilon \lambda \|v\|_2^2
\le
v^\top \K v.\]
This expands out to,
\begin{align}
    \frac{1 - \varepsilon}{\lambda_i} u_i^\top \Z^\top \Z u_i + \frac{1 - \varepsilon}{\lambda} q_i^\top \Z^\top \Z q_i + 2 \frac{1 - \varepsilon}{\sqrt{ \lambda_i \cdot \lambda} } u_i^\top \Z^\top \Z q_i &\le
    \varepsilon\left( \frac{\lambda}{\lambda_i} + 1 \right) + v^\top \K v \nonumber\\ 
    &\le 2\varepsilon + \frac{u_i^\top \K u_i}{\lambda_i} + \frac{q_i^\top \K q_i}{\lambda}\nonumber\\
    &= 2\varepsilon + 1 + \frac{q_i^\top \K q_i}{\lambda},\label{eq:cross-term-in-terms-Z}
\end{align}
where the first inequality above follows because $u_i^\top q_i = \frac{u_i^\top \Q_{\setminus t} \Z^\top \Z u_i}{\|\Q_{\setminus t} \Z^\top \Z u_i\|_2} = 0$ for every $i \in [t]$. The second inequality above also follows because $u_i^\top \K q_i = \lambda_i \cdot u_i^\top q_i = 0$.
Now note that, $u_i = \Q_t u_i$ for every $i \in [t]$, thus, by \cref{eq:head-spectral-approx}, $u_i^\top \Z^\top \Z u_i \ge \frac{1-\varepsilon}{1+\varepsilon} \cdot u_i^\top \K u_i = \frac{1-\varepsilon}{1+\varepsilon} \cdot \lambda_i$. Furthermore, using the fact that $p_i = \Q_{\setminus t} q_i$ along with \cref{eq:tail-spectral-approx}, we have $q_i^\top \Z^\top \Z q_i \ge q_i^\top \K q_i - 3\varepsilon r \lambda \|q_i\|_2^2 = q_i^\top \K q_i - 3\varepsilon r \lambda$. Plugging these inequalities into \cref{eq:cross-term-in-terms-Z} gives,
\[ 2 \frac{1 - \varepsilon}{\sqrt{ \lambda_i \cdot \lambda} } u_i^\top \Z^\top \Z q_i \le 9\varepsilon + \varepsilon \cdot \frac{q_i^\top \K q_i}{\lambda} \le 10 \varepsilon,\]
where the second inequality follows because $q_i$ lies in the column span of $\Q_{\setminus t}$, thus $q_i^\top \K q_i \le \lambda_{t+1} \le \lambda$. Therefore,
\[ (u_i^\top \Z^\top \Z q_i)^2 \le 26 \varepsilon^2 \cdot \lambda \cdot \lambda_i. \]
Plugging into \cref{cross-term-norm-sum} gives:
\[ \left\| \BSigma^{-1}_t \U_{t}^\top \Z^\top \Z \Q_{\setminus t} \right\|_F^2 \le \sum_{i=1}^t 26 \varepsilon^2 \lambda \le 52 \varepsilon^2 r \lambda, \]
where for the second inequality we used the fact that $t \le 2r$.
Returning to \cref{eq:cross-term-cauchy-schwarz} gives,
\[ \trace (\Y \K \K^{\dagger} \Q_{t} \Z^\top \Z \Q_{\setminus t} ) \le 8 \varepsilon \cdot \sqrt{ r \lambda \cdot \trace (\Y \K \Y) } \le 8 \varepsilon \cdot \trace (\Y \K \Y), \]
where the second inequality follows from the fact that $r \lambda = \sum_{i=r+1}^n \lambda_i \le \trace (\Y \K \Y)$. 

\paragraph{Final Bound.}
Finally by combining the bounds we obtained for {\bf Head Terms}, {\bf Tail Terms}, and {\bf Cross Term} and applying the fact that $r \lambda = \sum_{i=r+1}^n \lambda_i \le \trace (\Y \K \Y)$, we find that,
\[ \left| \trace ( \Y \Z^\top \Z  \Y ) - 
 \trace (\Y \K \Y) \right|
\le 4\varepsilon \trace (\Y \Q_t \K \Q_t \Y) + 4 \varepsilon r \lambda + 8 \varepsilon \trace (\Y \K \Y) \le 16 \varepsilon \trace (\Y \K \Y). \]
The proof of \cref{thm:main-proj-cost-preserv-approx} follows by substituting $\varepsilon/16$ in place of $\varepsilon$ in all the bounds above.
\end{proofof}

\section{Spectral Approximation of Dot-product Kernels}\label{sec-proof-thm-spectral-truncated-random-features}

In this section we first provide formal statement of \cref{thm-spectral-truncated-random-features} and prove it.

\begin{theorem*}[Formal statement of \cref{thm-spectral-truncated-random-features}]
	Suppose \cref{assumpt-growth-integral-kappa} holds for a dot-product kernel $\kappa(\inner{x,y})$. 
	Given $\X=[x_1, \dots, x_n] \in \RR^{d \times n}$ for $d\geq3$, assume that $\max_{j\in[n]}\norm{x_j} \le r$. Let $\K$ be the kernel matrix corresponding to $\kappa(\cdot)$ and $\X$.
	For any $0 < \lambda \le \norm{\K}_{\op}$ and $\varepsilon , \delta >0$ let $s_\lambda$ be the statistical dimension of $\K$.
	Also let $\Z$ be the proposed random features in \cref{eq:zonal-kernel-random-features} with $q= \max \left\{ d , 3.7 r^2\beta_\kappa, r^2\beta_\kappa + \frac{d}{2} \log \frac{3r^2 \beta_\kappa}{d} + \log \frac{C_\kappa n}{ \varepsilon \lambda} \right\}$, $s= \max \left\{ \frac{d}{2} , 3.7 r^2 \beta_\kappa, \frac{r^2\beta_\kappa}{4} + \frac{1}{2}\log \frac{C_\kappa n}{ \varepsilon \lambda} \right\}$ and  $m = \frac{5q^2}{4\varepsilon^{2}} \cdot {q + d-1 \choose q} \cdot \log \frac{16s_\lambda}{\delta}$. Then, with probability at least $1-\delta$, $\Z^\top \Z$ is an $(\varepsilon,\lambda)$-spectral approximation to $\K$ as per \cref{eq-spec-approx-truncated-features}. Furthermore, $\Z$ can be computed in time $\bigo((ms/q) \cdot \nnz{\X})$.
\end{theorem*}


\begin{proof}
We first show that the low-degree GZK $k_{q,s}(x,y)$ corresponding to the radial functions $h_\ell(\cdot)$ defined in \cref{def-monomial-coeff-dot-prod-kernel}, tightly approximates the kernel $\kappa(\inner{x,y} )$ on every pair of points $x, y$ in our dataset for $q= \max \left\{ d , 3.7 r^2\beta_\kappa, r^2\beta_\kappa + \frac{d}{2} \log \frac{3r^2 \beta_\kappa}{d} + \log \frac{C_\kappa n}{ \varepsilon \lambda} \right\}$ and $s= \max \left\{ \frac{d}{2} , 3.7 r^2 \beta_\kappa, \frac{r^2\beta_\kappa}{4} + \frac{1}{2}\log \frac{C_\kappa n}{ \varepsilon \lambda} \right\}$.
By \cref{lem-dot-prod-spherical-harmonic-expansion} and triangle inequality we have,
\begin{align}
    \left| k_{q,s}(x,y) - \kappa(\langle x , y \rangle) \right| &\le \left| \sum_{\ell > q} \left( \sum_{i=0}^\infty \widetilde{h}_{\ell, i}(\|x\|) \cdot \widetilde{h}_{\ell , i}(\|y\|) \right) \cdot P_{d}^{\ell}\left( \frac{\langle x, y \rangle}{\|x\| \cdot \|y\|} \right) \right|\label{eq:high-degree-term-kappa-truncation}\\
    &\qquad+ \left| \sum_{\ell =0}^q \left( \sum_{i \ge s} \widetilde{h}_{\ell , i}(\|x\|) \cdot \widetilde{h}_{\ell , i}(\|y\|) \right) \cdot P_{d}^{\ell}\left( \frac{\langle x, y \rangle}{\|x\| \cdot \|y\|} \right) \right|,\label{eq:low-degree-high-order-term-kappa-truncation}
\end{align}
where $\widetilde{h}_{\ell, i}(\cdot)$ is defined as per \cref{eq:gegenbauer-expansion-coeff-computation}.
We bound the terms in \cref{eq:high-degree-term-kappa-truncation} and \cref{eq:low-degree-high-order-term-kappa-truncation} separately. We first show that the coefficients of the monomials $\widetilde{h}_{\ell, i}(\cdot)$ in \cref{eq:gegenbauer-expansion-coeff-computation} decay exponentially as a function of $i$ and $\ell$. Since $\kappa(\langle x, y \rangle)$ is a valid kernel function, the derivative $\kappa^{(\ell+2i)}(0)$ must be non-negative for any $\ell$ and $i$ \cite{schoenberg1988positive}.
Using the fact that $\alpha_{\ell,d} \le \frac{(\ell + d-1)!}{\ell!(d-1)!}$ along with \cref{assumpt-growth-integral-kappa}, we find the following bound for any $t\ge 0$,
\begin{equation}
    0 \le \widetilde{h}_{\ell, i}(t) \le \sqrt{ \frac{C_\kappa \cdot \beta_\kappa^{\ell+2i} \cdot \Gamma(\frac{d}{2}) }{\sqrt{\pi} \cdot (d-1)!}\cdot \frac{(\ell+d-1)!}{ 2^\ell \cdot \ell! \cdot (2i)!} \cdot \frac{\Gamma(i+\frac{1}{2}) }{\Gamma(i+\ell+ \frac{ d}{2})} } \cdot t^{\ell + 2i}. \label{eq:bound-kappa-funuc-gegenbauer-coeff-bounds}
\end{equation}
Now, using \cref{eq:bound-kappa-funuc-gegenbauer-coeff-bounds}, we can bound the term in \cref{eq:high-degree-term-kappa-truncation} as follows,
\begin{align*}
    \left| \sum_{\ell > q} \left( \sum_{i=0}^\infty \widetilde{h}_{\ell, i}(\|x\|) \widetilde{h}_{\ell, i}(\|y\|) \right) \cdot P_{d}^{\ell}\left( \frac{\langle x, y \rangle}{\|x\|  \|y\|} \right) \right| &\le \sum_{\ell > q} \left( \sum_{i=0}^\infty \widetilde{h}_{\ell, i}(\|x\|) \cdot \widetilde{h}_{\ell, i}(\|y\|) \right)\\
    &\le \sum_{\ell > q} \sum_{i=0}^\infty  \frac{C_\kappa \cdot \beta_\kappa^{\ell+2i} \cdot \Gamma(\frac{d}{2}) }{\sqrt{\pi} \cdot (d-1)!}\cdot \frac{(\ell+d-1)!}{ 2^\ell \cdot \ell! \cdot (2i)!} \cdot \frac{\Gamma(i+\frac{1}{2}) }{\Gamma(i+\ell+ \frac{ d}{2})} \cdot r^{2\ell+4i} \\
    &= \frac{C_\kappa \cdot \Gamma(\frac{d}{2}) }{\sqrt{\pi} \cdot (d-1)!} \sum_{\ell > q} \frac{(\ell+d-1)!}{ 2^\ell \cdot \ell!} \cdot \sum_{i=0}^\infty \frac{\Gamma(i+\frac{1}{2}) }{\Gamma(i+\ell+ \frac{ d}{2})} \cdot \frac{(r^2 \beta_\kappa)^{\ell+2i}}{(2i)!} \\
    &\le \frac{C_\kappa \cdot \Gamma(\frac{d}{2})  \cdot e^{r^2 \beta_\kappa}}{4 \cdot (d-1)!} \sum_{\ell > q} \frac{(\ell+d-1)!}{ 2^\ell \cdot \ell!} \cdot \frac{(r^2 \beta_\kappa)^\ell}{\Gamma(\ell + \frac{d}{2})} 
\end{align*}
where in the last line above we used the fact that $\frac{\Gamma(i+\frac{1}{2}) }{\Gamma(i+\ell+ \frac{ d}{2})}$ is a decreasing function of $i$ and the sum $\sum_{i=0}^\infty \frac{(r^2 \beta_\kappa)^{\ell+2i}}{(2i)!} = \cosh(r^2 \beta_\kappa) \le 0.57 e^{r^2 \beta_\kappa}$.
Now we can further upper bound the above as follows
\begin{align}
    \left| \sum_{\ell > q} \left( \sum_{i=0}^\infty \widetilde{h}_{\ell, i}(\|x\|) \widetilde{h}_{\ell, i}(\|y\|) \right) \cdot P_{d}^{\ell}\left( \frac{\langle x, y \rangle}{\|x\|  \|y\|} \right) \right| &\le \frac{C_\kappa \cdot \Gamma(\frac{d}{2}) \cdot e^{r^2 \beta_\kappa}}{4 \cdot (d-1)!} \sum_{\ell > q} \frac{(\ell+d-1)!}{ 2^\ell \cdot \ell!} \cdot \frac{(r^2 \beta_\kappa)^\ell}{\Gamma(\ell + \frac{d}{2})} \nonumber\\ 
    &\le \frac{C_\kappa \cdot \Gamma(\frac{d}{2})  \cdot e^{r^2 \beta_\kappa}}{4 \cdot (d-1)!} \cdot \sum_{\ell > q} \frac{1}{ \ell^{\ell - \frac{d}{2}} } \cdot \left(\frac{ e \cdot r^2 \beta_\kappa }{ 2}\right)^\ell \cdot \left(1 + \frac{d-1}{\ell} \right)^{\frac{d}{2}} \nonumber\\
    &\le \frac{C_\kappa \cdot \Gamma(\frac{d}{2}) \cdot 2^{\frac{d}{2}} \cdot e^{r^2 \beta_\kappa}}{5 \cdot (d-1)!} \cdot \sum_{\ell > q} \frac{1}{ \ell^{\ell - \frac{d}{2}} } \cdot \left(\frac{ e \cdot r^2 \beta_\kappa }{ 2}\right)^\ell \nonumber\\
    &\le \frac{C_\kappa \cdot e^{r^2 \beta_\kappa}}{20 (d/2)^{d/2}} \cdot \sum_{\ell > q} \frac{1}{ \ell^{\ell - \frac{d}{2}} } \cdot \left(\frac{ e \cdot r^2 \beta_\kappa }{ 2}\right)^\ell \nonumber\\
    &\le \frac{C_\kappa \cdot e^{r^2 \beta_\kappa}}{20} \cdot \left(\frac{ e \cdot r^2 \beta_\kappa }{ d}\right)^{d/2} \sum_{\ell > q}  \left(\frac{ e \cdot r^2 \beta_\kappa }{2 \ell}\right)^{\ell - \frac{d}{2}} \nonumber\\
    &\le \frac{\varepsilon\lambda}{20n}. \label{bound:high-degree-term-kappa-truncation}
\end{align}

Similarly we upper bound the term in \cref{eq:low-degree-high-order-term-kappa-truncation}
\begin{align}
    \left| \sum_{\ell =0}^q \left( \sum_{i \ge s} \widetilde{h}_{\ell, i}(\|x\|) \widetilde{h}_{\ell, i}(\|y\|) \right) \cdot P_{d}^{\ell}\left( \frac{\langle x, y \rangle}{\|x\| \|y\|} \right) \right| &\le \sum_{\ell =0}^q \left( \sum_{i \ge s} \widetilde{h}_{\ell, i}(\|x\|) \cdot \widetilde{h}_{\ell, i}(\|y\|) \right)\nonumber\\
    &\le \sum_{\ell =0}^q \sum_{i \ge s} \frac{C_\kappa \cdot \beta_\kappa^{\ell+2i} \cdot \Gamma(\frac{d}{2}) }{\sqrt{\pi} \cdot (d-1)!}\cdot \frac{(\ell+d-1)!}{ 2^\ell \cdot \ell! \cdot (2i)!} \cdot \frac{\Gamma(i+\frac{1}{2}) }{\Gamma(i+\ell+ \frac{ d}{2})} \cdot r^{2\ell + 4i}\nonumber\\
    &\le \frac{C_\kappa \cdot \Gamma(\frac{d}{2}) }{\sqrt{\pi} \cdot (d-1)!} \sum_{i=s}^\infty \frac{\Gamma(i+\frac{1}{2}) \cdot (r^2 \beta_\kappa)^{2i}}{(2i)!} \sum_{\ell=0}^q \frac{(\ell+d-1)! \cdot (r^2 \beta_\kappa)^\ell }{ 2^\ell \cdot \ell! \cdot \Gamma(i+\ell+ \frac{ d}{2})} \nonumber\\
    &\le \frac{C_\kappa \cdot \Gamma(\frac{d}{2}) }{5 \cdot (d-1)!} \sum_{i=s}^\infty \frac{\Gamma(i+\frac{1}{2}) \cdot (r^2 \beta_\kappa)^{2i}}{(2i)!} \cdot  \frac{(d-1)! \cdot e^{\frac{r^2 \beta_\kappa}{2} } }{ \Gamma(i +\frac{ d}{2})} \nonumber\\
    &= \frac{C_\kappa \cdot \Gamma(\frac{d}{2}) \cdot e^{\frac{r^2 \beta_\kappa}{2} } }{5} \sum_{i=s}^\infty \frac{\Gamma(i+\frac{1}{2}) \cdot (r^2 \beta_\kappa)^{2i}}{(2i)! \cdot \Gamma(i +\frac{ d}{2})} \nonumber\\
    &\le \frac{C_\kappa \cdot e^{\frac{r^2 \beta_\kappa}{2} } }{20} \sum_{i=s}^\infty \left(\frac{e \cdot r^2 \beta_\kappa}{2i} \right)^{2i} \nonumber\\
    &\le \frac{\varepsilon \lambda}{20 n}. \label{bound:low-degree-high-order-term-kappa-truncation}
\end{align}

Thus, by combining \cref{bound:high-degree-term-kappa-truncation} and \cref{bound:low-degree-high-order-term-kappa-truncation}, we find that for every pair of points $x, y \in \RR^d$ with $\norm{x} , \norm{y} \le r$ the following holds,
\[ \left| k_{q,s}(x,y) - \kappa(\langle x , y \rangle) \right| \le \frac{\varepsilon \lambda}{10n}. \]
Therefore, if we let $\widetilde{\K} \in \RR^{n \times n}$ be the kernel matrix corresponding to kernel function ${k}_{s,q}(\cdot)$ and dataset $\X$, then we have the following,
\[ \left\| \widetilde{\K} - \K \right\|_F \le \frac{\varepsilon \lambda}{10}. \]

Now we let $\Z \in \RR^{(m \cdot s) \times n}$ be the random features matrix as in \cref{def-random-feature-construction} corresponding to the kernel function ${k}_{q,s}(x,y)$. 
The bound on the number of features given in \cref{thm:main-spectral-approx} for the kernel function ${k}_{q,s}(x,y)$ is upper bounded by,
\[ \sum_{\ell=0}^q \alpha_{\ell,d} \min \left\{ \frac{\pi^2 (\ell+1)^2}{6\lambda} \sum_{j \in [n]} \left\| h_{\ell}(\|x_j\|) \right\|^2 , s \right\} \le 1.1 s \cdot {q + d-1 \choose q}. \]
Thus by plugging this bound into \cref{thm:main-spectral-approx} we get that,
\[ (1 - 8\varepsilon/10) \cdot (\widetilde{\K} + \lambda \I)
\preceq
\Z^\top \Z  + \lambda \I
\preceq
(1 + 8\varepsilon/10) \cdot (\widetilde{\K} + \lambda \I).\]
The fact that $\left\| \widetilde{\K} - \K \right\|_F \le \frac{\varepsilon \lambda}{10}$ gives the lemma.

\paragraph{Runtime.}
The runtime of computing the features in \cref{def-random-feature-construction} is equal to the time to compute $\X^\top w_j$ for all $j \in [m]$ along with the time to evaluate the polynomials $P_d^\ell(t)$ at $mn$ different values of $t$ for all $\ell \in [q]$. These operations can be done in total time $\bigo\left( m \cdot \nnz{\X} \right)=\bigo\left( (ms/q) \cdot \nnz{\X} \right)$. Note that, to compute these random features we also need to evaluate the derivatives of function $\kappa(t)$ at zero (up to order $q$), however this is just a one time computation and does not need to be repeated for each data-point, thus, we can assume that this time would not depend on $n$ or $m$ or $d$ and is negligible compared to $\bigo\left( (ms/q) \cdot \nnz{\X} \right)$.

\end{proof}

\section{Spectral Approximation to Gaussian Kernel}\label{sec-thm-spectral-approx-Gaussian}
In this section we prove \cref{thm-spectral-approx-Gaussian}.
\thmspectralapproxgaussian*

\begin{proof}
We first show that the low-degree GZK $g_{q,s}(x,y)$ corresponding to the radial functions $h_\ell(\cdot)$ defined in \cref{def-monomial-coeff-gauss-kernel}, tightly approximates the Gaussian kernel $g(x,y)$ on every pair of points $x, y$ in our dataset for $q= \max \left\{ 3.7r^2, \frac{d}{2} \log\frac{2.8 (r^2 + \log\frac{n}{\varepsilon\lambda} + d)}{d} + \log \frac{n}{ \varepsilon \lambda} \right\}$ and $s= \max \left\{ \frac{d}{2}, 3.7r^2, \frac{1}{2} \log \frac{n}{\varepsilon \lambda} \right\}$.
By \cref{eq-gauss-kernel-zonal-expansion} and triangle inequality we have the following,
\begin{align}
    \left| g_{q,s}(x,y) - g( x , y) \right| &\le \left| \sum_{\ell > q} \left( \sum_{i=0}^\infty \widetilde{h}_{\ell ,i}(\|x\|) \cdot \widetilde{h}_{\ell ,i}(\|y\|) \right) \cdot P_{d}^{\ell}\left( \frac{\langle x, y \rangle}{\|x\| \cdot \|y\|} \right) \right|\label{eq:high-degree-term-gauss-truncation}\\
    &\qquad+ \left| \sum_{\ell =0}^q \left( \sum_{i \ge s} \widetilde{h}_{\ell ,i}(\|x\|) \cdot \widetilde{h}_{\ell ,i}(\|y\|) \right) \cdot P_{d}^{\ell}\left( \frac{\langle x, y \rangle}{\|x\| \cdot \|y\|} \right) \right|,\label{eq:low-degree-high-order-term-gauss-truncation}
\end{align}
where $\widetilde{h}_{\ell ,i}(\cdot)$ is defined as in the statement of \cref{eq-gauss-kernel-zonal-expansion}.
We bound the terms in \cref{eq:high-degree-term-gauss-truncation} and \cref{eq:low-degree-high-order-term-gauss-truncation} separately.
By Cauchy–Schwarz inequality and the fact that $\widetilde{h}_{\ell ,i}(\|x\|)$ and $\widetilde{h}_{\ell ,i}(\|y\|)$ are non-negative, we can bound \cref{eq:high-degree-term-gauss-truncation} as follows,
\begin{align*}
    \left| \sum_{\ell > q} \left( \sum_{i=0}^\infty \widetilde{h}_{\ell ,i}(\|x\|) \widetilde{h}_{\ell ,i}(\|y\|) \right)  P_{d}^{\ell}\left( \frac{\langle x, y \rangle}{\|x\| \|y\|} \right) \right| &\le \left| \sum_{\ell > q} \left( \sum_{i=0}^\infty \widetilde{h}_{\ell ,i}(\|x\|) \cdot \widetilde{h}_{\ell ,i}(\|y\|) \right) \right|\\
    &\le \sqrt{ \sum_{\ell > q} \sum_{i=0}^\infty |\widetilde{h}_{\ell ,i}(\norm{x})|^2 \cdot \sum_{\ell > q} \sum_{i=0}^\infty |\widetilde{h}_{\ell ,i}(\norm{y})|^2 }.
\end{align*}
Now we can bound the term $\sum_{\ell > q} \sum_{i=0}^\infty |\widetilde{h}_{\ell ,i}(\norm{x})|^2$, using the definition of $\widetilde{h}_{\ell ,i}(\cdot)$, as follows,
\small
\begin{align*}
    \sum_{\ell > q} \sum_{i=0}^\infty |\widetilde{h}_{\ell ,i}(\norm{x})|^2 &\le \frac{ \Gamma(\frac{d}{2}) }{\sqrt{\pi} \cdot (d-1)!}\cdot \sum_{\ell> q} \frac{(\ell+d-1)!}{ 2^\ell \cdot \ell!} \cdot \sum_{i=0}^\infty   \frac{\Gamma(i+\frac{1}{2}) }{\Gamma(i+\ell+ \frac{ d}{2})} \cdot \frac{\|x\|^{2\ell + 4i} e^{-\|x\|^2}}{(2i)!} \\
    &\le \frac{ \Gamma(\frac{d}{2}) }{4 \cdot (d-1)!}\cdot \sum_{\ell> q} \frac{(\ell+d-1)!}{ 2^\ell \cdot \ell!} \cdot  \frac{\|x\|^{2\ell} }{\Gamma(\ell+ \frac{ d}{2})} \\
    &\le \frac{1}{4} \sum_{\ell> q} \left( \frac{e \cdot \ell}{d} \right)^{\frac{d}{2}} \cdot \frac{\|x\|^{2\ell} }{2^\ell \cdot \ell!} \\
    &\le \frac{1}{4} \sum_{\ell> q} \left( \frac{e \cdot \ell}{d} \right)^{\frac{d}{2}} \cdot \left(\frac{e \cdot r^2}{2 \ell}\right)^{\ell} \\
    &\le \frac{\varepsilon \lambda}{20 n}.
\end{align*}
\normalsize

Similarly, we can show $\sum_{\ell > q} \sum_{i=0}^\infty |\widetilde{h}_{\ell ,i}(\norm{y})|^2 \le \frac{\varepsilon\lambda}{20n}$, thus, the term in \cref{eq:high-degree-term-gauss-truncation} is bounded by,
\begin{equation}\label{bound:high-degree-term-gauss-truncation}
\left| \sum_{\ell > q} \left( \sum_{i=0}^\infty \widetilde{h}_{\ell ,i}(\|x\|) \cdot \widetilde{h}_{\ell ,i}(\|y\|) \right) \cdot P_{d}^{\ell}\left( \frac{\langle x, y \rangle}{\|x\| \cdot \|y\|} \right) \right| \le \frac{\varepsilon\lambda}{20n}.
\end{equation}

Now we upper bound the term in \cref{eq:low-degree-high-order-term-gauss-truncation} using Cauchy–Schwarz inequality as follows,
\small
\begin{align*}
    \left| \sum_{\ell=0}^q \left( \sum_{i \ge s} \widetilde{h}_{\ell ,i}(\|x\|) \cdot \widetilde{h}_{\ell ,i}(\|y\|) \right)  P_{d}^{\ell}\left( \frac{\langle x, y \rangle}{\|x\| \|y\|} \right) \right| &\le \left| \sum_{\ell =0}^q \left( \sum_{ i \ge s} \widetilde{h}_{\ell ,i}(\|x\|) \cdot \widetilde{h}_{\ell ,i}(\|y\|) \right) \right|\\
    &\le \sqrt{ \sum_{\ell =0}^q \sum_{i \ge s} |\widetilde{h}_{\ell ,i}(\|x\|)|^2 \cdot \sum_{\ell=0}^q \sum_{i \ge s} |\widetilde{h}_{\ell ,i}(\|y\|)|^2 }.
\end{align*}
\normalsize
Now we can bound the term $\sum_{\ell =0}^q \sum_{i \ge s} |\widetilde{h}_{\ell ,i}(\|x\|)|^2$, using the definition of $\widetilde{h}_{\ell ,i}(\cdot)$, as follows,
\small
\begin{align*}
    \sum_{\ell =0 }^q \sum_{i \ge s}  |\widetilde{h}_{\ell ,i}(\|x\|)|^2 &\le \frac{ \Gamma(\frac{d}{2}) }{\sqrt{\pi} \cdot (d-1)!}\cdot \sum_{i \ge s} \frac{\Gamma(i+\frac{1}{2})}{(2i)!} \sum_{\ell =0}^q \frac{(\ell+d-1)!}{ 2^\ell \cdot \ell!} \cdot    \frac{\|x\|^{2\ell + 4i} e^{-\|x\|^2} }{\Gamma(i+\ell+ \frac{ d}{2})} \\
    &\le \frac{ \Gamma(\frac{d}{2}) }{\sqrt{\pi} \cdot (d-1)!}\cdot \sum_{i \ge s} \frac{\Gamma(i+\frac{1}{2})}{(2i)!} \cdot \frac{(d-1)!}{\Gamma(i + \frac{d}{2})} \sum_{\ell =0}^q \frac{ \|x\|^{2\ell + 4i} e^{-\|x\|^2} }{ 2^\ell \cdot \ell!}\\
    &\le \frac{ \Gamma(\frac{d}{2}) }{\sqrt{\pi}}\cdot \sum_{i \ge s} \frac{\Gamma(i+\frac{1}{2})}{\Gamma(i + \frac{d}{2}) \cdot (2i)!} \cdot  \|x\|^{4i} e^{-\|x\|^2/2}\\
    &\le \frac{e^{-\|x\|^2/2}}{5} \sum_{i \ge s} \left( \frac{e \cdot \|x\|^2}{2i} \right)^{2i}\\
    &\le \frac{\varepsilon \lambda}{20n}.
\end{align*}
\normalsize

Similarly, we can show $\sum_{\ell =0}^q \sum_{i \ge s} |\widetilde{h}_{\ell ,i}(\|y\|)|^2 \le \frac{\varepsilon\lambda}{20n}$, thus, the term in \cref{eq:low-degree-high-order-term-gauss-truncation} is bounded by,
\begin{equation}\label{bound:low-degree-high-order-term-gauss-truncation}
\left| \sum_{\ell =0}^q \left( \sum_{i \ge s} \widetilde{h}_{\ell ,i}(\|x\|) \cdot \widetilde{h}_{\ell ,i}(\|y\|) \right) \cdot P_{d}^{\ell}\left( \frac{\langle x, y \rangle}{\|x\| \cdot \|y\|} \right) \right| \le \frac{\varepsilon\lambda}{20n}.
\end{equation}
Thus by combining \cref{bound:high-degree-term-gauss-truncation} and \cref{bound:low-degree-high-order-term-gauss-truncation}, we find that for every pair of points $x, y \in \RR^d$ with $\norm{x} , \norm{y} \le r$ the following holds,
\[ \left| {g}_{q,s}(x,y) - g( x , y ) \right| \le \frac{\varepsilon \lambda}{10n}. \]
Therefore, if we let $\widetilde{\K} \in \RR^{n \times n}$ be the kernel matrix corresponding to kernel function ${g}_{s,q}(\cdot)$ and dataset $\X$, then,
\[ \left\| \widetilde{\K} - \K \right\|_F \le \frac{\varepsilon \lambda}{10}. \]

Now we let $\Z \in \RR^{(m \cdot s) \times n}$ be the random features matrix as in \cref{def-random-feature-construction} corresponding to the kernel function ${g}_{q,s}(x,y)$. 
The bound on the number of features given in \cref{thm:main-spectral-approx} for the kernel function ${g}_{q,s}(x,y)$ is upper bounded by,
\[ \sum_{\ell=0}^q \alpha_{\ell,d} \min \left\{ \frac{\pi^2 (\ell+1)^2}{6\lambda} \sum_{j \in [n]} \left\| h_{\ell}(\|x_j\|) \right\|^2 , s \right\} \le 1.1 s \cdot {q + d-1 \choose q}. \]
Thus by plugging this bound into Theorem~\ref{thm:main-spectral-approx} we get that,
\[ (1 - 8\varepsilon/10) \cdot (\widetilde{\K} + \lambda \I)
\preceq
\Z^\top \Z  + \lambda \I
\preceq
(1 + 8\varepsilon/10) \cdot (\widetilde{\K} + \lambda \I).\]
Using the fact that $\left\| \widetilde{\K} - \K \right\|_F \le \frac{\varepsilon \lambda}{10}$ gives the lemma.

\paragraph{Runtime.}
The runtime of computing the features in \cref{def-random-feature-construction} is equal to the time to compute $\X^\top \sigma_j$ for all $j \in [m]$ along with the time to evaluate the polynomials $P_d^\ell(t)$ at $mn$ different values of $t$ for all $\ell \in [q]$. These operations can be done in total time $\bigo\left( m \cdot \nnz{\X} \right) =\bigo\left( (ms/q) \cdot \nnz{\X} \right)$

\end{proof}

\section{Experimental Details} 

\subsection{Details on Kernel Ridge Regression} \label{sec-exp-details-krr}
For kernel ridge regression, we use $4$ real-world datasets, e.g., Earth Elevation\footnote{\url{https://github.com/fatiando/rockhound}}, \coo\footnote{\url{https://db.cger.nies.go.jp/dataset/ODIAC/}}, Climate\footnote{\url{http://berkeleyearth.lbl.gov/}} and Protein\footnote{\url{https://archive.ics.uci.edu/}}. 
For Elevation, \coo, Climate datasets, each data point is represented by a (latitude, longitude) pair. 
We convert the location values into the 3D-Cartesian coordinates (i.e., $\mathbb{S}^2$). 
In addition, both \coo and climate datasets contain 12 different temporal values. and append the temporal one if they exist. 
For Protein dataset, each data point is given by 10-dimensional features. We consider the first $9$ features as training data and the final feature as label. We also normalize those features so that each feature has zero mean and 1 standard deviation. 
For all datasets, we randomly split 90\% training and 10\% testing, and find the ridge parameter via the 2-fold cross-validation on the training set. For all kernel approximation methods, we set the final feature dimension to $m=1{,}024$.

\subsection{Details on Kernel $k$-means Clustering} \label{sec-exp-details-kmeans}
For kernel $k$-means clustering, we use $6$ UCI classification datasets\footnote{\url{http://persoal.citius.usc.es/manuel.fernandez.delgado/papers/jmlr/data.tar.gz}}.  We normalize the inputs by this $l_2$ norms so that they are on the unit sphere.   In addition, we use the $k$-means clustering algorithm from an open-source scikit-learn\footnote{\url{https://scikit-learn.org/}} package ($\mathtt{sklearn.cluster.KMeans}$) where initial seed points are chosen by $k$-mean++ initialization~\cite{arthur2006k}. The number of clusters is set to the number of classes of each dataset and the number of features are commonly set to $m=512$.

\end{document}